\PassOptionsToPackage{table}{xcolor}
\documentclass[11pt]{article}
\usepackage[margin=1.2in]{geometry}
\usepackage{iftex}
\usepackage{natbib}
\usepackage{microtype}
\usepackage{titlesec}
\usepackage{hyperref}
\usepackage{amsmath,amssymb}

\titleformat*{\section}{\normalfont\bfseries\large}
\titleformat*{\subsection}{\normalfont\bfseries\normalsize}

\titlespacing*{\section}{0pt}{1.7em}{0.8em}
\usepackage{xcolor}
\definecolor{linkblue}{HTML}{5E17EB}  

\hypersetup{
  colorlinks=true,
  linkcolor=linkblue,
  citecolor=linkblue,
  urlcolor=linkblue,
  pdfborder={0 0 0}
}

\ifPDFTeX
  \usepackage[T1]{fontenc}
  \usepackage[utf8]{inputenc}
  \usepackage{tgpagella}
  \usepackage[scale=0.9]{tgheros}
  \usepackage[scaled=0.85]{beramono}
\else
\usepackage{fontspec}
\fi

\usepackage{titlesec}
\titleformat{\paragraph}[runin]
  {\normalfont\bfseries}
  {}{0pt}{}

\titlespacing*{\paragraph}
  {0pt}{1em}{1em}

\usepackage{hyperref}
\usepackage{url}

\usepackage{graphicx}
\usepackage{hyperref}
\usepackage{textcomp}
\usepackage{url}
\usepackage{amsthm}
\usepackage{amssymb}
\usepackage{booktabs}
\usepackage{listings}
\usepackage{float}
\usepackage{wrapfig}
\usepackage{lipsum}
\usepackage{subcaption}
\newtheorem{theorem}{Theorem}

\newtheorem{proposition}[theorem]{Proposition}

\newtheorem{definition}[theorem]{Definition}

\usepackage{graphicx}
\usepackage{hyperref}
\usepackage{url}
\usepackage{amsthm}
\usepackage{amssymb}
\usepackage{fvextra}
\usepackage{algorithm}
\usepackage{algpseudocode}
\usepackage{changepage}
\usepackage{listings}
\usepackage[most]{tcolorbox}
\tcbuselibrary{listings,breakable,skins}

\DefineVerbatimEnvironment{ExampleBlock}{Verbatim}{
  fontsize=\scriptsize,
  breaklines=true,
  breakanywhere=true,
  frame=single,
  framesep=2mm
}

\usepackage{listings}
\usepackage[most]{tcolorbox}
\tcbuselibrary{listings,breakable}

\usepackage{listings}
\usepackage[most]{tcolorbox}

\tcbuselibrary{listings,breakable,skins}

\input{theorem_boxing_style.tex}

\usepackage{amsmath,amsfonts,bm}

\def\eqref#1{equation~\ref{#1}}

\def\1{\bm{1}}

\DeclareMathAlphabet{\mathsfit}{\encodingdefault}{\sfdefault}{m}{sl}
\SetMathAlphabet{\mathsfit}{bold}{\encodingdefault}{\sfdefault}{bx}{n}

\usepackage[most]{tcolorbox}
\tcbuselibrary{listings,breakable,skins}
\tcbset{
  modeltranscript/.style={
    enhanced,
    listing only,
    breakable,
    colback=white,
    colframe=black!25,
    boxrule=0.4pt,
    sharp corners,
    boxsep=0pt,
    left=2.5mm,
    right=2.5mm,
    top=2mm,
    bottom=2mm,
    before skip=10pt,
    after skip=12pt,
    colbacktitle=white,
    coltitle=black,
    fonttitle=\normalfont\footnotesize,
    toptitle=2.5mm,
    bottomtitle=2.5mm,
    lefttitle=2.5mm,
    righttitle=2.5mm,
    titlerule=0.4pt,
    listing options={
      basicstyle=\ttfamily\scriptsize,
      breaklines=true,
      breakatwhitespace=false,
      breakautoindent=false,
      breakindent=0pt,
      columns=fullflexible,
      keepspaces=true,
      showstringspaces=false,
      tabsize=2,
      literate=*%
        {A}{{A\allowbreak}}1%
        {C}{{C\allowbreak}}1%
        {G}{{G\allowbreak}}1%
        {T}{{T\allowbreak}}1%
        {N}{{N\allowbreak}}1%
        {Å}{{\AA}}1%
        {°}{{\ensuremath{^\circ}}}1%
        {±}{{\ensuremath{\pm}}}1%
        {²}{{\ensuremath{^2}}}1%
        {³}{{\ensuremath{^3}}}1%
        {·}{{\ensuremath{\cdot}}}1%
        {×}{{\ensuremath{\times}}}1%
        {α}{{\ensuremath{\alpha}}}1%
        {β}{{\ensuremath{\beta}}}1%
        {γ}{{\ensuremath{\gamma}}}1%
        {–}{{--}}1%
        {—}{{---}}1%
        {…}{{\ldots}}1%
        {⁵}{{\ensuremath{^5}}}1%
        {⁶}{{\ensuremath{^6}}}1%
        {⁺}{{\ensuremath{^+}}}1%
        {⁻}{{\ensuremath{^-}}}1%
        {₂}{{\ensuremath{_2}}}1%
        {₃}{{\ensuremath{_3}}}1%
        {₄}{{\ensuremath{_4}}}1%
        {₆}{{\ensuremath{_6}}}1%
        {₉}{{\ensuremath{_9}}}1%
        {→}{{\ensuremath{\rightarrow}}}1%
        {↔}{{\ensuremath{\leftrightarrow}}}1%
        {−}{{\ensuremath{-}}}1%
        {√}{{\ensuremath{\surd}}}1%
        {≈}{{\ensuremath{\approx}}}1%
        {≠}{{\ensuremath{\neq}}}1%
        {✓}{{\ensuremath{\checkmark}}}1
    },
  },
}
\newtcblisting{modelinput}{
  modeltranscript,
  title={Model Input}
}
\newtcblisting{modelinputsystem}{
  modeltranscript,
  title={Model Input (System)}
}
\newtcblisting{modelinputuser}{
  modeltranscript,
  title={Model Input (User)}
}
\newtcblisting{modeloutput}{
  modeltranscript,
  title={Model Output}
}

\usepackage{vcell}
\usepackage{amsfonts}       
\usepackage{nicefrac}       
\usepackage{xcolor}         
\usepackage{enumitem}
\usepackage{multirow,makecell}
\usepackage[table]{xcolor}
\usepackage{siunitx}
\usepackage{caption}
\usepackage[most]{tcolorbox}
\usepackage[normalem]{ulem}

\definecolor{targetbg}{HTML}{E8F4FC}
\definecolor{otherbg}{HTML}{FFF0E5}
\definecolor{supportbg}{HTML}{E8F8EF}
\definecolor{conflictbg}{HTML}{FDECEC}
\definecolor{unusedbg}{HTML}{F3F3F3}

\definecolor{correctgreen}{HTML}{159957}
\definecolor{incorrectred}{HTML}{D83A3A}
\definecolor{trainblue}{HTML}{276FBF}
\definecolor{testorange}{HTML}{D96C16}
\usepackage{vcell}
\newcolumntype{Y}{>{\raggedright\arraybackslash}X}

\newcommand{\trainlabel}{%
  \textcolor{trainblue}{\textbf{\scriptsize TRAIN}}%
}

\newcommand{\testlabel}{%
  \textcolor{testorange}{\textbf{\scriptsize TEST}}%
}

\definecolor{claimbg}{HTML}{F7F7F7}
\definecolor{claimbar}{HTML}{4A5568}

\newtcolorbox{claimbox}{
  enhanced,
  breakable,
  colback=claimbg,
  colframe=claimbg,
  boxrule=0pt,
  leftrule=3pt,
  colbacktitle=claimbg,
  arc=2pt,
  left=7pt,
  right=7pt,
  top=6pt,
  bottom=6pt,
  before skip=8pt,
  after skip=8pt,
  borderline west={3pt}{0pt}{claimbar}
}
\usepackage{multirow}
\usepackage[table]{xcolor}
\usepackage{amsmath,amssymb,amsfonts,mathtools}
\newcommand{\IAPtext}{Information Abundance Paradox}
\newcommand{\IAP}{\hyperref[sec_paradox]{\IAPtext}}
\newcommand{\IAPemph}{\hyperref[sec_paradox]{\emph{\IAPtext}}}

\usepackage[table]{xcolor}
\usepackage{array}
\usepackage{makecell}

\definecolor{headerblue}{HTML}{DCEEFF}
\definecolor{headerorange}{HTML}{FCE3C6}
\definecolor{supportgreen}{HTML}{EAF7EA}
\definecolor{conflictred}{HTML}{FCEAEA}
\definecolor{lightgray}{HTML}{F3F3F3}

\definecolor{phi4k}{HTML}{D96A80}   
\definecolor{phi128k}{HTML}{5A9ED6} 

\definecolor{olmo8k}{HTML}{a8dfc0}
\definecolor{olmo65k}{HTML}{c8b7ef}
\definecolor{sa_only}{HTML}{FFE070}

\usepackage[table]{xcolor}
\usepackage{array}
\usepackage{tabularx}
\usepackage{multirow}
\usepackage{makecell}
\usepackage[most]{tcolorbox}

\definecolor{targetblue}{HTML}{9EDCFF}
\definecolor{otherorange}{HTML}{FFBC8A}
\definecolor{supportgreen}{HTML}{86F0C2}
\definecolor{conflictred}{HTML}{FFB0B0}
\definecolor{answergreen}{HTML}{00B85A}
\definecolor{answerred}{HTML}{FF3038}
\definecolor{lightgray}{HTML}{F7F7F7}

\newtcbox{\tablelabel}[1][]{%
  on line,
  arc=8pt,
  boxrule=0pt,
  left=5pt,
  right=5pt,
  top=3pt,
  bottom=3pt,
  colback=gray!15,
  #1
}

\usepackage[most]{tcolorbox}

\definecolor{takeawaybg}{HTML}{FAFAFA}
\definecolor{takeawayframe}{HTML}{D9D9D9}

\newtcolorbox{takeawaybox}{
  enhanced,
  breakable,
  colback=takeawaybg,
  colframe=takeawayframe,
  boxrule=0.4pt,
  arc=4pt,
  left=6pt,
  right=6pt,
  top=5pt,
  bottom=5pt,
  before skip=6pt,
  after skip=6pt,
  fontupper=\small,
}

\newcommand{\Risk}{\mathcal{R}}

\usepackage[table]{xcolor}
\usepackage{array}

\definecolor{targetblue}{HTML}{DCEEFF}
\definecolor{otherorange}{HTML}{FCE3C6}
\definecolor{supportgreen}{HTML}{EAF7EA}
\definecolor{conflictred}{HTML}{FCEAEA}
\definecolor{unusedgray}{HTML}{F3F3F3}

\newcolumntype{C}[1]{>{\centering\arraybackslash}p{#1}}
\newcolumntype{L}[1]{>{\raggedright\arraybackslash}p{#1}}

\usepackage{etoolbox}

\makeatletter
\AfterEndEnvironment{definition}{\@afterindentfalse}
\AfterEndEnvironment{theorem}{\@afterindentfalse}
\AfterEndEnvironment{lemma}{\@afterindentfalse}
\AfterEndEnvironment{proposition}{\@afterindentfalse}
\AfterEndEnvironment{corollary}{\@afterindentfalse}
\makeatother

\definecolor{palviolet}{HTML}{5E17EB}      
\definecolor{palvioletlight}{HTML}{9E74F3}
\definecolor{palvioletwash}{HTML}{F1ECFD}
\definecolor{palnavy}{HTML}{1F5A96}        
\definecolor{palnavylight}{HTML}{4A90C2}
\definecolor{palcoral}{HTML}{D9603F}       
\definecolor{palcoralwash}{HTML}{FBE6DF}
\definecolor{palgold}{HTML}{B8862F}        
\definecolor{palsandwash}{HTML}{FBF1DE}
\definecolor{palorange}{HTML}{E8833A}      
\definecolor{palorangedark}{HTML}{B4561A}
\colorlet{phi4k}{palnavylight}
\colorlet{phi128k}{palnavy}
\colorlet{olmo8k}{palorange}
\colorlet{olmo65k}{palorangedark}
\colorlet{kzero}{palcoral}
\colorlet{keight}{palnavy}
\colorlet{ffncolor}{palcoral}
\colorlet{sacolor}{palgold}
\colorlet{targetbg}{palvioletwash}
\colorlet{otherbg}{palsandwash}
\colorlet{supportbg}{palvioletwash}
\colorlet{conflictbg}{palcoralwash}
\colorlet{trainblue}{palviolet}
\colorlet{testorange}{palorange}
\colorlet{correctgreen}{palviolet}
\colorlet{incorrectred}{palcoral}
\title{\vspace{-1.5em}Context Scaling is \textbf{\textit{No}} Free Lunch}

\author{
    \small
    Arda Uzunoglu$^{1}$,
    Benjamin Van Durme$^{1}$
    Daniel Khashabi$^{1}$ \\[0.6em]
    \small
    $^{1}$Johns Hopkins University\\[1em]
    \small
    \href{https://github.com/ardauzunoglu/information-abundance-paradox}{GitHub}
    \quad \quad
    \href{https://huggingface.co/collections/ardauzunoglu/information-abundance-paradox}{Hugging Face}
}

\date{}

\begin{document}
\maketitle

\begin{adjustwidth}{0.35in}{0.35in}
\begin{center}
    \bfseries Abstract
\end{center}
\vspace{0.3em}

\noindent
Large language models are increasingly trained and deployed with long contexts that span documents, code repositories, and interaction histories. This scaling implicitly treats context as a free lunch, assuming that training on longer contexts can only help the model by exposing it to richer evidence.
We challenge this view by studying how the context window shapes a model's \emph{mode of learning}, shifting it between parametric \emph{internalization} and \emph{contextualization}. 
We propose the \IAPemph, which hypothesizes that abundant relevant information in the training context can reduce the incentive to encode that information \textit{parametrically}, thereby
increasing reliance on \textit{context}. 
In pretraining with long documents,
increasing the context window improves language modeling, natural language
understanding, and closed-book MCQA only up to an intermediate optimum, after which performance consistently declines. 
In supervised fine-tuning, more task-relevant train-time context improves performance with supporting context, but reduces robustness when context is absent or misleading at test time.
Our analysis suggests that this behavior arises when longer context provides a
lower complexity solution.
Mechanistically, training with informative context 
shifts
gradient pressure from feed-forward networks, often linked to parametric knowledge,
toward attention modules, and causal interventions show that this shift increases reliance on context during inference.
Our findings therefore suggest that context scaling is no free lunch, even when high-quality long-context data is abundant.
\end{adjustwidth}

\begin{figure}[H]
    \centering
    \includegraphics[width=\linewidth]{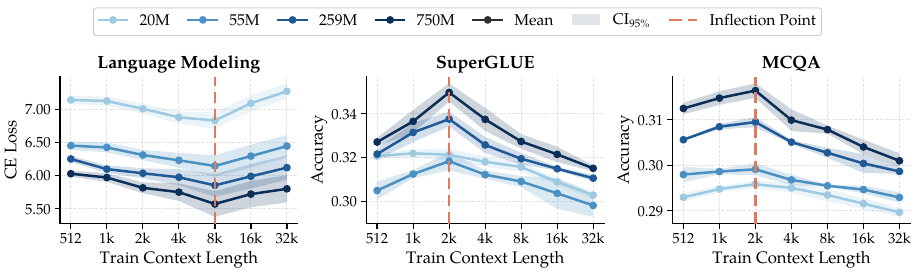}
    \vspace{-2em}
    \caption{\textbf{Performance improves up to an intermediate optimum as pretraining context window grows.} Across model sizes, SuperGLUE and MCQA follow an inverted-U pattern, while language modeling shows a corresponding U-shaped loss curve.}
    \label{fig:pt_res}
\end{figure}

\section{Introduction}

Large language models (LLMs) are increasingly deployed in settings where useful evidence spans long documents~\citep{longbench}, codebases~\citep{repobench, swebench}, and extended interaction histories~\citep{webarena}.
In response, training practices have extended context windows to increasingly large scales~\citep{yarn, chen2023extendingcontextwindowlarge}, in some cases reaching millions of tokens~\citep{gemini15, longrope}. Yet it remains unclear what, if anything, fundamentally limits this trajectory. This raises a basic question:
\begin{questionbox}
Is context scaling a free lunch, limited only by the supply of \mbox{long-context} data?\end{questionbox}

The empirical trajectory of long-context scaling reflects an implicit assumption that sufficiently plentiful long-context data will provide the signal needed for continued progress.
Therefore, when long-context training underperforms, prior work often attributes the
bottleneck to the scarcity of naturally long, high-quality documents~\citep{chen2025ladm} and focuses on data-driven recipes as the path forward~\citep{xiong2024effective, fu2024data, gao2024train, gao2025next}. 

\begin{wrapfigure}[29]{R}{0.475\textwidth}
    \centering
    \includegraphics[width=\linewidth]{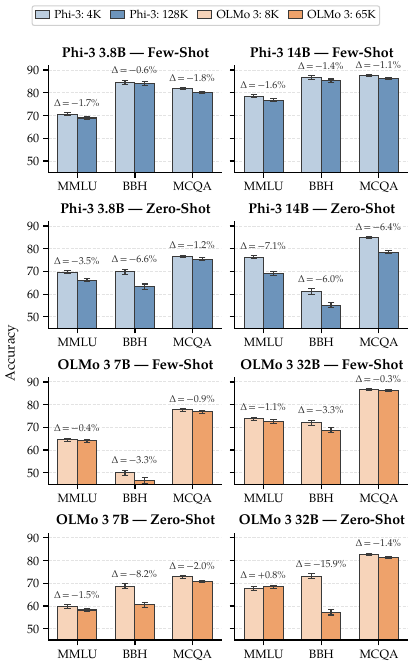}
    \vspace{-2em}
    \caption{%
    \textbf{Longer-context Phi-3 and OLMo 3 variants underperform their short-context counterparts.}
    Error bars denote 95\% CI.}
    \label{fig:phi_results}
\end{wrapfigure}
\paragraph{A motivating observation.} This data-centric view overlooks the possibility that the context window is not a neutral conduit for data. 
Phi-3~\citep{phi3} and OLMo 3~\citep{olmo3} models provide motivating examples, as their long-context (\textcolor{phi128k}{128K} for Phi-3 and \textcolor{olmo65k}{65K} for OLMo 3) variants consistently underperform their short-context (\textcolor{phi4k}{4K} for Phi-3 and \textcolor{olmo8k}{8K} for OLMo 3) counterparts in few-shot and zero-shot settings (\autoref{fig:phi_results}). Since the long-context variants are trained on \textit{more} and \textit{longer} data, this gap cannot be attributed to data quality alone~\citep{phi3, olmo3}. 
This suggests that context scaling is not a free lunch, raising the question of what long-context training costs the model.

\vspace{-0.5em}
\paragraph{Two modes of learning.} Training a language model by minimizing next token cross-entropy can be viewed as a form of compression~\citep{witten87,deletang2024language}, whereby reusable regularities in the training corpus are distilled into the model parameters~\citep{Grunwald_2019}.

Critically, such regularities enable two sources of predictive power: knowledge stored in the parameters~\citep{petroni2019language,roberts2020knowledge}, and information supplied in context~\citep{rag,gpt3}. This distinction suggests two \emph{modes of learning}~\citep{pan2023what,lin2024dual,ciphers}, in which a model either \emph{internalizes} task-relevant information by encoding it in its parameters or \emph{contextualizes} that information by learning to use evidence supplied in context. These two sources of predictive power offer alternative channels for prediction, with their relative use shaped by the training dynamics~\citep{wang2023investigating} and the informativeness of the context~\citep{anand2025dual}. 

\paragraph{Our hypothesis.}

Building on this view, we propose \IAP, which posits that information-rich long-context training can shift a model's learned strategy away from internalization and toward contextualization. As a result, when context is absent, missing, or insufficient at inference time, performance can degrade relative to models trained with less informative context, whether shorter or equally long but less relevant (\S\ref{sec_hypothesis}).

\paragraph{Our evidence.}

We test \IAP{} in pretraining and supervised fine-tuning. In pretraining, longer context windows yield \emph{inverted-U} performance, with language modeling, natural language understanding, and closed-book MCQA improving up to an intermediate optimum before performance degrades (\S\ref{subsec_pt}). In supervised fine-tuning, task-relevant context improves accuracy under informative context, but reduces robustness when context is absent or misleading (\S\ref{sec_sft}). We then provide a theoretical account of these findings, showing that longer contexts can reduce the amount of task information that must be stored in the weights (\S\ref{sec_theory}). Lastly, our mechanistic analyses reveal that this phenomenon arises when longer context provides a lower complexity solution (\S\ref{sec_solution_complexity}) and shifts update pressure from feed-forward networks to attention modules (\S\ref{sec_grad_all}). This shift also carries over to inference, where models trained with task-relevant context allocate more attention to context tokens (\S\ref{sec_attn}). Together, our findings show that context scaling is no free lunch, since longer training contexts can fundamentally alter how much knowledge models can draw from their parameters and how strongly they depend on information supplied in context.

\section{The Information Abundance Paradox}
\label{sec_paradox}
\label{sec_hypothesis}

The train-time context window controls information available to a model during training.
Increasing the context window is intended to expand the information a model can exploit, while ideally preserving short-context competence~\citep{xiong2023effectivelongcontextscalingfoundation, longrope}.

We argue that this gain comes at a cost: when task-relevant information is available in context during training, the model can reduce loss by using it directly rather than by encoding the same information in its weights. We call this hypothesis the \emph{Information Abundance Paradox}:

\begin{definition}[Information Abundance Paradox]
\label{def:iap}
When task-relevant information is made available through the training context,
the model can reduce loss by using that information directly rather than by encoding it in its parameters. Consequently, this
can shift the model's mode of learning away from parametric
internalization and toward contextualization.
\end{definition}

Here, \emph{information abundance} refers to task-relevant information available
within a context, and not the amount of training data or the entropy of the corpus. The central implication is that
context length is not merely a data-delivery mechanism. It can also determine whether the model learns to rely on information stored parametrically or on information supplied through the context.

We refer to the behavioral manifestation of this shift in models as \textbf{\emph{context addiction}}, where a model trained with informative context performs \textit{well when useful context is available} but deteriorates when that context is \emph{absent} or \emph{misleading}. Thus, context addiction provides an observable test of the \emph{Information Abundance Paradox} through robustness under absent or misleading context.

\section{Testing the Information Abundance Paradox}
\label{sec_exp}

Our hypothesis concerns \textit{the amount of task-relevant information} available in the training context. We operationalize information abundance differently across the two training regimes. In pretraining, where relevance is difficult to control, we vary the length of coherent document spans as a proxy for information availability (\S\ref{subsec_pt}). In supervised fine-tuning, where relevance can be controlled directly, we fix the context length and vary the amount of task-relevant information to isolate its role (\S\ref{sec_sft}).

\subsection{Pretraining with Varying Context Length}
\label{subsec_pt}

We first test the \IAP{} in pretraining, where the context window controls how much within-document evidence is available during next token prediction. We \textit{vary} the training context window while \textit{fixing} the token budget, data, model configuration and optimization setup. We then evaluate how the training context window affects downstream task performance.

\paragraph{Model architecture.}
We pretrain language models at four scales, 20M, 55M, 259M, and 750M parameters. All models use the Llama-2 architecture and tokenizer~\citep{llama2}, RoPE~\citep{rope} for positional encoding, a standard causal attention mask, and the next token cross-entropy objective. We verify that our findings are robust to alternative positional encodings in App.~\ref{app_positional_encoding}.

\paragraph{Data.}
We train on 10B tokens drawn from a subset of Project Gutenberg~\citep{pg_corpus}. We retain documents containing at least 65536 tokens, so that increasing the training context window exposes longer coherent within-document spans rather than merely increasing cross-document packing.

\paragraph{Training setup.}
For each model scale, we sweep the training context window over seven choices,
$W \in$ \{$512, 1024, \ldots, 32768$\}, in powers of two. For each $W$, the corpus is partitioned into non-overlapping sequences of exactly $W$ tokens, with documents concatenated only as needed to fill complete sequences. Loss is computed over every token in the sequence. The global batch size is fixed at approximately 1.05M tokens, corresponding to 9537 optimization steps for every variant. Longer-context variants therefore use fewer sequences per batch. Within each model scale, all context window variants share the same model initialization, random seed, optimizer, and training hyperparameters. We train all models using the \texttt{nanotron} framework~\citep{nanotron}, with full model configurations and hyperparameters provided in App.~\ref{sec_nl_pretrain}. \textit{Our comparison is therefore token-and-update-matched, isolating context window effects under a fixed token budget}.\footnote{The comparison is not FLOP-matched, since longer windows incur higher attention cost as self-attention scales quadratically with sequence length~\citep{transformers}.}

\paragraph{Evaluation.}
We evaluate the final checkpoint of each model in the zero-shot setting across
three complementary testbeds: (i)~a language modeling suite
(e.g., LAMBADA~\citep{paperno2016lambada}, WikiSPAN~\citep{wikispan}, and Penn Treebank~\citep{ptb}), measuring next token prediction over
natural language text; (ii)~SuperGLUE~\citep{wang2020superglue}, measuring general language understanding; and
(iii)~a suite of closed-book MCQA benchmarks (e.g., ARC~\citep{clark2018think}, CommonsenseQA~\citep{talmor2019commonsense}, PIQA~\citep{bisk2019piqa}), targeting parametric knowledge. Evaluation uses each benchmark as provided, with no context beyond the task input. We report cross-entropy loss on the
language modeling suite and accuracy on SuperGLUE and MCQA tasks. For
multiple-choice benchmarks, answers are selected by length-normalized log probability of the candidate answer text tokens, excluding the option label.
We provide the complete benchmark list in App.~\ref{app_benchmarks} and motivate their inclusion in our evaluation.

\paragraph{Results.}
As shown in \autoref{fig:pt_res}, performance on SuperGLUE and MCQA follows an \emph{inverted-U} pattern as the pretraining context window grows, while language modeling loss follows a corresponding \emph{U-shaped} pattern. Models initially benefit from longer
contexts up to a testbed-dependent inflection point, as performance peaks
around 2048 tokens for SuperGLUE and MCQA and around 8192 tokens for language
modeling. Beyond these points, further increases in context length progressively
erode the earlier gains. The presence of an inflection point across all three testbeds suggests that the effect reflects a systematic effect of pretraining context length, with statistically significant inflections confirmed across evaluation suites and model scales (App.~\ref{app_sig_pretraining}).
Importantly, greater model
capacity does not eliminate the long-context degradation, as the same qualitative pattern persists across all four model
scales, from 20M to 750M parameters. One possible explanation for the earlier optimum on SuperGLUE and MCQA is that the optimal train-time context length depends partly on the length distribution of the evaluation instances, for which we provide supporting evidence in App.~\ref{sec_dataset_decomp}. A control that matches sequence length but removes long-range context eliminates both the gain and the degradation, indicating that both require informative long-range context (App.~\ref{app:stitched}). \uline{Taken together, these results show that, under matched token budgets, context length is not a free scaling axis, since longer training windows can degrade performance beyond an intermediate optimum.}

\subsection{Supervised Fine-Tuning with Varying Context Informativeness}
\label{sec_sft}

We now test the \IAP{} in supervised fine-tuning for knowledge-rich tasks, where we \textit{fix} the context budget and \textit{vary} the task-relevant information in context. This setting isolates the role of informative train-time context while holding the task, model, and training process fixed.

\paragraph{Setup.}
We fine-tune \texttt{Qwen3 0.6B, 1.7B, 4B, 8B, 14B}~\citep{qwen3} with LoRA~\citep{hu2021loralowrankadaptationlarge} on four MMLU-Pro domains~\citep{mmlupro}: Health, Economics, Law, and Psychology. Since MMLU-Pro provides only test splits, we partition the examples within each domain into 80\% training and 20\% evaluation sets. For each question, we construct a fixed context budget of $n=8$ documents and vary the number of target-domain documents $k \in \{0,4,8\}$, where the remaining $8-k$ documents are drawn from the other domains (App.~\ref{app_sft_doc_gen}). Therefore, we only change the domain composition of the eight documents across $k$, controlling the fraction of \emph{informative context} while holding the other aspects of fine-tuning fixed. 
We report document length statistics in App.~\ref{app_sft_doc_gen}.
We apply LoRA adapters to both attention modules and feed-forward networks, and compute the loss only on answer text tokens. Training details are provided in App.~\ref{app_sft_details}.

\paragraph{Evaluation.}
We evaluate each model on held-out questions from the domain used for its fine-tuning, without in-context demonstrations. While training varies the number of target-domain documents, evaluation fixes documents to the target domain and varies whether they support the correct or an incorrect answer (see \autoref{tab:sft-document-example} in App.~\ref{app_sft_doc_gen}). Accordingly, we conduct evaluation under three test-time context conditions: (i) \emph{supporting context}, where the prepended documents 
support the correct answer;
(ii) \emph{conflicting context}, where the prepended documents 
support an incorrect answer;
and (iii) \emph{no context}, where the model receives only the question and answer choices. We use deterministic generation without sampling and report exact-match accuracy against the ground truth answer.

\begin{figure}[H]
    \centering
    \includegraphics[width=\linewidth]{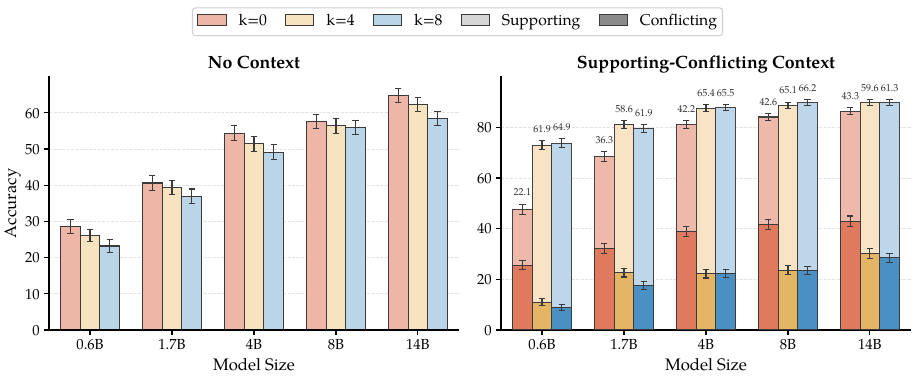}
    \vspace{-2em}
    \caption{\textbf{Task-relevant train-time context induces context addiction.} Increasing target-domain documents from \textcolor{kzero}{$k=0$} to \textcolor{keight}{$k=8$} improves Qwen3 models with supporting context, but reduces robustness without context (left) and with conflicting context (right). Error bars denote 95\% CI.
    }
    \label{fig:sft_res}
\end{figure}

\paragraph{Results.}
\autoref{fig:sft_res} reports domain-averaged results, with
domain-specific breakdowns deferred to App.~\ref{app_sft_full_results}. The results reveal that
train-time context affects model behavior primarily through its relevance to the
task. Comparing fine-tuning runs with $k=0, 4, 8$, we
find that increasing the number of target-domain
documents strengthens
performance when supporting context is available at test time. 
However, this improvement comes with significantly reduced no-context accuracy and significantly greater vulnerability to conflicting context across model sizes and domains (App.~\ref{app_sig_sft_trends}).
This suggests that the observed shift is driven not by additional tokens alone, but by the task-relevant information they provide. The no-context drop also persists when test inputs match the $k=0$ training format, indicating that it does not stem from a change in input format alone (App.~\ref{app:format-matched}). \uline{Taken together, these results show that train-time context is not neutral background information, since its relevance can determine whether models internalize the task or contextualize it through the in-context evidence.}

\section{A Theoretical Account of Information Abundance Paradox}
\label{sec_theory}
Our pretraining experiments (\S\ref{subsec_pt}) show that increasing the training context window improves performance only up to an intermediate optimum, beyond which performance on evaluation inputs without long context degrades. 

We now formalize why this tradeoff is possible. The key idea is that context and weights can act as alternative carriers  of task-relevant information. A longer context gives the predictor an additional channel through which task information can be accessed, and therefore can reduce the minimum amount of task information that must be stored in the weights.

\paragraph{Setup.}
Let \(\tau \sim P_{\mathcal T}\) denote a discrete latent task variable indexing the data-generating distribution, treating the observed corpus as one realization from a broader population of all possible corpora.\footnote{For instance, our pretraining experiments use one filtered subset of Project Gutenberg, whereas many comparably sized subsets could have been sampled from the same underlying collection.}
For each context size \(k\), let \((X^{(k)},Y)\sim P_\tau^{(k)}\) denote an
input-output pair with input \(X^{(k)}\) of context size \(k\). 
Let \(W\sim\Pi(\cdot\mid\tau)\) denote the learned weights induced by training on data from task \(\tau\), with randomness due to data sampling and optimization. A language model with context size \(k\) defines a conditional predictor
\(q_k\), which induces predictions according to $\hat Y\sim q_k(\cdot\mid X^{(k)},W)$.
Let \(\ell:\mathcal Y\times\widehat{\mathcal{Y}}\to\mathbb R_+\) be a task loss
(e.g., binary token error) between the target \(Y\) and prediction \(\hat Y\).
We define the risk
$\Risk_k(\Pi,q_k)=\mathbb E[\ell(Y,\hat Y)],$
where the expectation is over \((X^{(k)},Y)\sim P_\tau^{(k)}\),
\(W\sim \Pi(\cdot\mid \tau)\), and
\(\hat Y\sim q_k(\cdot\mid X^{(k)},W)\). Thus, \(\Risk_k\) measures the expected task loss of the predictor
given access to context of size \(k\).
We quantify task-specific information stored in the weights by
\(I(W;\tau)\) computed under the joint distribution induced by \(P_{\mathcal T}\) and
\(\Pi(W\mid\tau)\). Equivalently,
\(I(W;\tau)=H(\tau)-H(\tau\mid W)\), so it measures the reduction in uncertainty about the
task after observing the learned weights.
While we vary the context size \(k\), we hold fixed the latent task distribution,
the underlying data-generating process, and the model architecture.

\begin{definition}[Parametric information frontier]
For risk threshold $\rho$, define
\[
\mathcal I_k(\rho)
=
\inf_{\Pi,\,q_k} I(W;\tau)
\qquad
\text{such that}
\qquad
\Risk_k(\Pi,q_k)\le \rho.
\]
Thus, \(\mathcal I_k(\rho)\) is the minimum task information that must be stored in the
weights to attain risk at most \(\rho\) when prediction has access to the input \(X^{(k)}\) corresponding to context size \(k\).
\end{definition}

Without loss of generality, we assume the inputs are nested across context sizes through the packed token
stream, where for each \(k<m\), a shorter-context sequence is an aligned subwindow of a
longer-context sequence. Accordingly, there exists a measurable projection map \(T_{k,m}\)
with \(X^{(k)} = T_{k,m}(X^{(m)})\) almost surely. Here, \(T_{k,m}\) selects the
corresponding \(k\)-token block within the \(m\)-token window, matching standard
pretraining pipelines that partition the same token stream into fixed-length
windows at different context sizes~\citep{nanotron}.

\begin{proposition}[Monotonicity of the parametric information frontier]
\label{prop_monotonicity_main}
If \(X^{(k)}=T_{k,m}(X^{(m)})\) almost surely for \(k<m\), then
\[
\mathcal I_m(\rho)\le\mathcal I_k(\rho)
\qquad \text{for all } \rho.
\]
\end{proposition}

\begin{proof}[Proof sketch]
A predictor with access to \(X^{(m)}\) can simulate any \(k\)-context predictor \(q_k\)
by first applying the projection map \(T_{k,m}\) to recover \(X^{(k)}\), and then
using the resulting input in the \(k\)-context predictor. Therefore, every
feasible weight-predictor pair for the optimization problem defining
\(\mathcal I_k(\rho)\) induces a feasible pair for the corresponding problem at
context size \(m\) with the same weight channel. Thus, enlarging the context
weakly expands the feasible set of the constrained optimization problem without
increasing the task information stored in the weights. Taking the infimum over
feasible solutions yields \(\mathcal I_m(\rho)\le \mathcal I_k(\rho)\).
\end{proof}

This is an \textit{achievability} statement, showing that longer context \textit{can} 
reduce the minimum amount of task information stored in the weights to achieve
a target risk threshold.\footnote{App.~\ref{app_theory_details} gives the full proof.}
If longer train-time context reduces the task information encoded in the
weights, then removing or corrupting context leaves the predictor with less
parametric knowledge to fall back on. This can cause performance to deteriorate
when context is absent or misleading, yielding the behavioral signature of
\emph{context addiction} (\S\ref{sec_hypothesis}).

\section{Mechanisms Behind the Information Abundance Paradox}

\subsection{Solution Complexity}
\label{sec_solution_complexity}
We conduct a controlled synthetic pretraining study to identify when longer context induces context addiction. By holding fixed the task, objective, and data-generating process while varying the number of in-context demonstrations, we test whether context addiction emerges selectively when additional demonstrations provide a lower complexity training trajectory.

\paragraph{Setup.} We
test this prediction in a synthetic pretraining setting with four
tasks: (i) unary bitwise operations, (ii) string operations, (iii)
mod10 arithmetic, and (iv) Caesar cipher. For each task, we
train language models at three scales (0.3M, 1.5M, and 7.5M
parameters), varying the number of in-context demonstrations while holding fixed
the task, objective, and data-generating process.
\autoref{fig:synth_grad_norm} reports results for the largest model, with the full model-scale breakdown provided in App.~\ref{app_comp_full_results}. Task definitions and
training details are provided in App.~\ref{app_synth_pretrain}.

\paragraph{Evaluation.}
We evaluate each model under supporting and conflicting context, paralleling the setup used in \S\ref{sec_sft}. Test-time context contains the same number of in-context demonstrations as the model observed during training.
In the
supporting condition, in-context examples follow the correct task rule, while in the
conflicting condition, they follow a consistent but incorrect rule. For example, in a bitwise-negation task, supporting examples map $\neg 0 \rightarrow 1$ and $\neg 1 \rightarrow 0$, whereas conflicting examples consistently map $\neg 0 \rightarrow 0$ and $\neg 1 \rightarrow 1$.
\textit{The
supporting--conflicting gap therefore measures reliance on supplied context
rather than on a parametrically internalized rule}.

\begin{figure}[H]
    \centering
    \includegraphics[width=\linewidth]{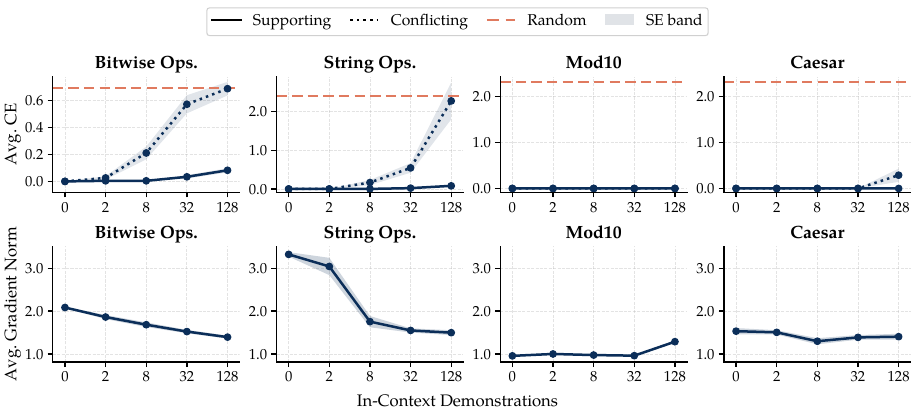}
    \vspace{-2em}
    \caption{\textbf{Context addiction emerges when longer train-time context enables lower complexity solutions.}
    \textbf{Top row:} Supporting--conflicting gaps grow for Bitwise and String operations, but remain stable for Mod10 and Caesar Cipher.
    \textbf{Bottom row:} Average training gradient norm decreases only in the context-addicted tasks.}
    \label{fig:synth_grad_norm}
\end{figure}

\paragraph{Results.}
\autoref{fig:synth_grad_norm} (top row) shows that longer train-time context does
not affect all tasks uniformly. For bitwise and string operations, the gap grows
steadily with context length, indicating increasing dependence on the
demonstrations as conflicting-context performance approaches the random baseline. In contrast,
mod10 remains unchanged across context lengths, while
Caesar cipher shows only a localized deviation at the longest context length. These trends persist when all models are evaluated with the same number of test-time demonstrations (App.~\ref{app:fixed-demos}).

To distinguish these regimes, we use the average training gradient norm as a comparative proxy for solution complexity, following a well-established line of work showing that gradient descent implicitly penalizes the training gradient norm \citep{barrett2022implicitgradientregularization, smith2021originimplicitregularizationstochastic} and that this implicit regularizer controls the geometric complexity of the learned function \citep{dherin2022neuralnetworkssimplesolutions} 
For context length $k$, with parameters $\theta_t^{(k)}$ at step $t$, we compute average gradient norm
$G_k = \frac{1}{T}\sum_{t=1}^T
\left\|\nabla_{\theta}\ell(\theta_t^{(k)}; B_t)\right\|_2,$
where $B_t$ is the training batch and $\ell$ is the final-query loss. At each step, we compute the global $\ell_2$ norm before gradient clipping by concatenating the gradients of all trainable parameters.

\autoref{fig:synth_grad_norm} aligns the gradient norm proxy (bottom) with the behavioral results (top) for the 7.5M model. \uline{Tasks with growing supporting--conflicting gaps also show decreasing average gradient norms as train-time context increases, whereas robust tasks show only noise-level changes or slight increases (see App.~\ref{app_sig_synth_grad} for significance tests).} This pattern supports the interpretation that context addiction emerges when demonstrations provide an easier optimization path than parametric internalization of the task rule.

\subsection{Module-Level Gradient Allocation}
\label{sec_grad_all}

Prior work suggests that feed-forward networks (FFNs) and self-attention (SA) heads play distinct roles in transformer language models, associating FFNs with factual and task-relevant knowledge stored in parameters, and SA heads with the selection and routing of information from context~\citep{geva2021ffn,dai2022knowledge,meng2023locating}.
Motivated by this distinction and prior work that uses relative gradient magnitudes to characterize transformer optimization dynamics~\citep{zhang2019improving,liu2020understanding,noci2022signal}, we analyze module-level gradient allocation using the FFN-to-SA gradient ratio to measure whether training updates concentrate more strongly in FFNs or SA heads as context length increases.

\paragraph{Module-wise gradient ratios.} 
For pretraining, we follow the setup in \S\ref{subsec_pt} and vary the context window size. For SFT, we follow the setup in \S\ref{sec_sft} and compare training with $k=8$ task-relevant documents against a no-context baseline that fine-tunes only on question-answer pairs.
For each training step, we first average the gradient norm within each module type across layers, compute the FFN-to-SA ratio, and then average this ratio across training steps. Because the architecture and parameterization are fixed, the FFN and SA gradients have fixed dimensionality, and, therefore, changes in the ratio reflect changes in relative gradient magnitude.
Accordingly, \textit{it tracks the relative optimization pressure between FFN-mediated internalization and SA-mediated contextualization}.

\begin{figure}[h!]
    \centering
    \begin{subfigure}{0.46\linewidth}
        \centering
        \includegraphics[width=\linewidth]{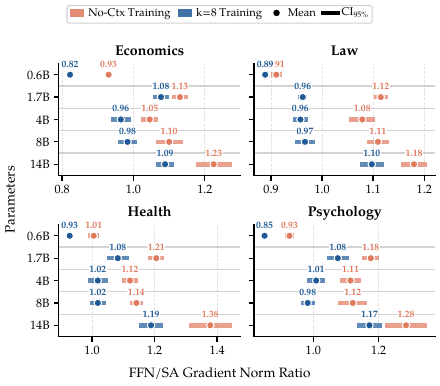}
        \caption{\textbf{SFT with task-relevant context shifts gradient pressure toward attention.}}
        \label{fig:grad_all}
    \end{subfigure}
    \hfill
    \begin{subfigure}{0.46\linewidth}
        \centering
        \includegraphics[width=\linewidth]{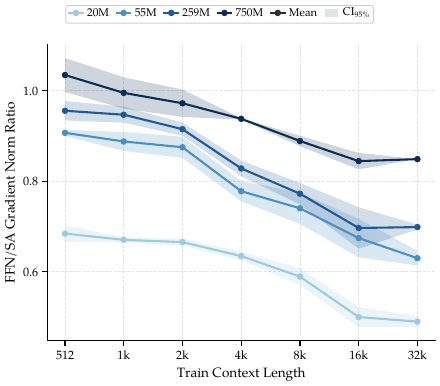}
        \caption{\textbf{Pretraining with longer context shifts gradient pressure toward attention.}}
        \label{fig:grad_all_pt}
    \end{subfigure}
    \caption{Informative train-time context lowers the FFN-to-SA gradient norm ratio.
    }
    \label{fig:grad_all_combined}
\end{figure}

\autoref{fig:grad_all} first shows this shift in supervised fine-tuning. When models are trained with $k=8$ target-domain documents, the FFN-to-SA gradient norm ratio is lower than in the no-context baseline across model sizes and domains. This indicates that adding task-relevant train-time context shifts relative gradient pressure away from FFNs and toward self-attention.
\autoref{fig:grad_all_pt} shows the analogous pattern in pretraining. As the training context window increases, the FFN-to-SA gradient norm ratio decreases across model scales, with the largest reductions appearing at longer context windows. Thus, both the controlled SFT setting and the pretraining setting point to the same qualitative change in module-level optimization dynamics. \uline{Taken together, these results show that context-rich training shifts gradient pressure from FFN-mediated internalization toward SA-mediated contextualization, supporting the view that longer train-time context changes the mode of learning. These shifts are statistically significant in 19 out of 20 SFT comparisons and at every scale in pretraining (App.~\ref{app_sig_grad_all}).}

\paragraph{Module-restricted fine-tuning.}
We further test this interpretation with module-restricted fine-tuning, in which models are trained with $k=8$ task-relevant documents while updating only FFN or only SA heads. To reduce compute, we perform these interventions on smaller Qwen3 models. 

\begin{figure}[H]
    \centering
    \includegraphics[width=\linewidth]{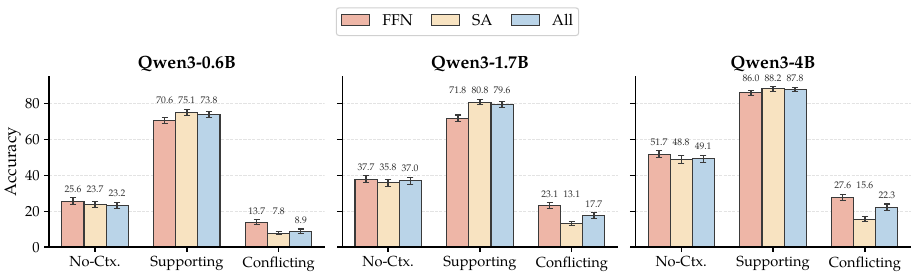}
    \caption{\textbf{FFN and SA updates causally control context reliance.} \textcolor{ffncolor}{FFN-only} tuning improves no-context robustness, while \textcolor{sacolor}{SA-only} tuning strengthens supporting-context performance but increases sensitivity to conflicting context. Error bars denote 95\% confidence intervals.}
    \label{fig:qwen3_ffn_attn_k8_avg}
\end{figure}

\autoref{fig:qwen3_ffn_attn_k8_avg} shows that FFN-only tuning yields stronger no-context performance and smaller degradation under conflicting context. In contrast, SA-only fine-tuning achieves significantly stronger performance with supporting context and significantly greater vulnerability to conflicting context across all three model sizes (App.~\ref{app_sig_module_restricted}). \uline{This module-restricted intervention provides causal evidence for the gradient analysis, showing that FFN-directed updates improve parametric robustness, whereas SA-directed updates increase reliance on the supplied context.} We next analyze whether this optimization shift is reflected at inference time via increased attention to context tokens (\S\ref{sec_attn}).

\subsection{Token-Level Attention Allocation}
\label{sec_attn}

Prior work has shown that transformer layers exhibit functional specialization, where early layers encode local and syntactic features, middle layers support contextual integration and retrieval, and later layers refine representations for prediction~\citep{tenney2019bert,jin2024cutting}. Motivated by this layer-wise structure, we ask whether the train-time shift in gradient allocation (\S\ref{sec_grad_all}) is reflected in inference-time attention to context tokens. We compare the fine-tuned models analyzed in \S\ref{sec_grad_all} under supporting-context evaluation. For a layer \(\ell\), head \(h\), and answer-token query position \(q\), let \(A^{\ell,h}_{q,j}\) denote the normalized attention weight from query position \(q\) to key position \(j\). We define the attention mass assigned to context tokens as \(\sum_{j \in \mathcal{C}} A^{\ell,h}_{q,j}\), where \(\mathcal{C}\) is the set of tokens in the prepended documents, excluding the question, answer choices, and generated answer tokens. For each layer and head, we average this quantity over generated answer-token positions and evaluation examples, and then report the maximum over heads within each layer. Our choice to report the maximum attention mass over heads is motivated by prior work using context-attention ratios to diagnose contextual grounding and showing that retrieval and in-context processing can be concentrated in sparse, specialized heads~\citep{voita2019analyzing,olsson2022incontext,chuang2024lookback,wu2025retrieval}.

\begin{figure}[h]
    \centering
    \includegraphics[width=\linewidth]{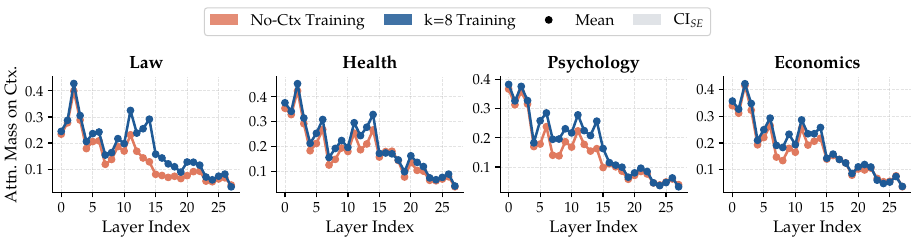}
    \caption{\textbf{SFT with task-relevant context increases inference-time attention to context.} Qwen3-1.7B fine-tuned with $k=8$ target-domain documents assigns more attention mass to context tokens.}
    \label{fig:attn_mass_qwen1}
\end{figure}

\autoref{fig:attn_mass_qwen1} reports results for \texttt{Qwen3-1.7B}; the other Qwen3 sizes show the same pattern (App.~\ref{sec_attn_all}).
\uline{Models trained with task-relevant context allocate significantly more attention to context tokens at test time across all four domains (App.~\ref{app_sig_attention}).}
The shift is concentrated in middle layers, consistent with their role in contextual integration.
This shows that increased train-time context availability changes the test-time computation, where the resulting models rely more heavily on the supplied context at prediction time.

\section{Related Work}
\label{sec_related_work}

\paragraph{Long-context training.}
Recent work has made long-context language modeling increasingly practical by extending usable context through positional extrapolation and RoPE scaling~\citep{alibi,chen2023extendingcontextwindowlarge,yarn,longrope}, sparse or recurrent attention~\citep{transformerxl,longformer,sparse_attn}, and retrieval or memory augmentation~\citep{retro,memorizingtransformers,infiniattention}. Other work adapts pretrained models to longer windows using continued pretraining and long-document upsampling~\citep{xiong2024effective}, domain and length balanced data mixtures~\citep{fu2024data}, and instruction tuning to shape long-context use~\citep{gao2024train,longskywork}. These methods primarily ask how models can acquire and use long-context capabilities while preserving existing performance. Short-to-long curricula recover long-context ability more
efficiently than training at maximum length
throughout~\citep{jin2023growlength,pouransari2024datasetdecomposition,skyladder},
suggesting that the training window may affect not only the attainable context
length, but also the training dynamics used to acquire it. Relatedly, prior work explains non-monotonic context scaling as a tradeoff between the predictive value of additional context and the difficulty of learning to use it with limited data and model capacity~\citep{shi2026intrinsic}. This account is complementary to ours, which focuses on how informative training context shifts learning from parametric knowledge toward context reliance.

Beyond effects on training efficiency and approximation difficulty, long-context adaptation can degrade short-context performance through representation drift and catastrophic forgetting~\citep{longred}, suggesting that naive context extension may trade one capability for another. A longer nominal window also does not by itself imply effective use of the added context, since nominal context length can overstate effective context length when long-range relative positions are undertrained~\citep{an2024doeseffectivecontextlength}. Relatedly, models systematically underuse information in the middle of long prompts~\citep{lostinthemiddle}, and long-context benchmarks reveal persistent failures in retrieval, aggregation, and reasoning over extended inputs~\citep{longbench,ruler,infinitebench,bianchi2025SmallerNeedles,byerly2026selfconsistencyfallsshortadverse}. Prior work thus primarily asks how to obtain, preserve, or evaluate long-context capability, but does not characterize how the training context window shapes the model's mode of learning. We address this gap by studying whether the training context window governs the allocation of task information
between parameters and context, thereby shifting models from parametric
internalization toward context-dependent computation.

We discuss related work on \emph{two modes of learning} and \emph{language modeling as compression} in App.~\ref{app_add_rel_work}. 

\section{Discussion and Conclusion}
\label{sec_discussion}

\paragraph{Implications.}
The \IAP{} challenges the assumption that context scaling is a free lunch that requires only more data. Our findings show that the context window can govern the model’s mode of learning, with longer windows shifting models toward context-supplied information and away from reusable parametric knowledge. Whether this shift is beneficial depends on the application.

\paragraph{Limitations.}
Our pretraining experiments are limited up to 750M parameters in pretraining. Testing the same phenomenon at larger scale remains an important next step for determining how the location and severity of the observed inflection points change with model size, data scale, and compute.

\paragraph{Conclusion.} Long-context processing is a powerful capability, but scaling it is no free lunch for language models. Our findings show that increasing train-time context can change the model's mode of learning, shifting it from parametric internalization toward contextualization. The takeaway is therefore not that long-context processing is undesirable, but that it should be understood through its role in mediating the tradeoff between context use and context independent competence.

\subsection*{AI Use Statement}

We used generative AI tools to assist with polishing the writing, implementation and debugging, and early-stage discussions on how to formulate and approach the theoretical account presented in this work. All AI-assisted outputs were reviewed and revised by the authors as appropriate. The authors take full responsibility for the content of the paper, including all text, code, analyses, and claims produced with the assistance of generative AI.







\bibliography{conference}
\bibliographystyle{conference}
\newpage
\appendix

\section{Evaluation of Phi-3 and OLMo 3 Models}
\label{app_phi3_eval}

We evaluate long and short-context variants from two model families, Phi-3~\citep{phi3} and OLMo 3~\citep{olmo3}. For Phi-3, we consider four instruction-tuned variants spanning two model scales, mini (3.8B parameters) and medium (14B parameters). Specifically, we evaluate \texttt{Phi-3-mini-4k-instruct}, \texttt{Phi-3-mini-128k-instruct}, \texttt{Phi-3-medium-4k-instruct}, and \texttt{Phi-3-medium-128k-instruct}. The short-context \textcolor{phi4k}{4K} variants correspond to the instruction-tuned checkpoints, whereas the long-context \textcolor{phi128k}{128K} variants are obtained through LongRoPE-based context extension with additional post-training~\citep{phi3}. For OLMo 3, we consider four base-model variants spanning two model scales, 7B and 32B parameters. Specifically, we evaluate the \textcolor{olmo8k}{8K} and \textcolor{olmo65k}{65K} variants for each scale. The short-context \textcolor{olmo8k}{8K} variants correspond to checkpoints from the end of Stage 2 pretraining, whereas the long-context \textcolor{olmo65k}{65K} variants correspond to checkpoints from the end of Stage 3 pretraining~\citep{olmo3}. We
evaluate these models on MMLU~\citep{mmlu}, BBH~\citep{bbh}, and the suite of MCQA benchmarks described in App.~\ref{app_mcqa}. We report accuracy as the evaluation metric for
all tasks, along with the binomial standard errors.

We use a generation-based evaluation protocol for both few-shot and
zero-shot settings using vLLM~\citep{vllm}. The model is prompted to generate an answer, and
we extract the predicted option letter from the generated output. For few-shot
evaluation, we follow the Phi-3 paper~\citep{phi3} and use the same
task-specific number of demonstrations:
MMLU 5, BBH 3, ARC-Easy 10, ARC-Challenge 10, CommonsenseQA 10,
BoolQ 0, OpenBookQA 10, PIQA 5, SocialIQA 5, SciQ 5, HellaSwag 5,
MedQA 2, TruthfulQA 10, and WinoGrande 5. For zero-shot evaluation, we
set \(k=0\) for all tasks. All evaluations use deterministic decoding.

We use the following multiple-choice prompt template:
\begin{quote}
\small
\texttt{Question: <question>} \\
\texttt{Options:} \\
\texttt{A. <choice A>} \\
\texttt{B. <choice B>} \\
\texttt{...} \\
\texttt{Answer:}
\end{quote}

\clearpage

\section{Positional Encoding Ablations}
\label{app_positional_encoding}

Our main pretraining experiments (\S\ref{subsec_pt}) use RoPE~\citep{rope} for positional encoding. Because the choice of positional encoding may influence performance across context lengths, we test whether the observed trends depend on this architectural choice. To this end, we compare RoPE with ALiBi~\citep{alibi} and LongRoPE~\citep{longrope} using the 259M parameter model under the pretraining setting of \S\ref{subsec_pt}. For ALiBi, we replace RoPE with distance dependent attention biases while retaining the same training setup. For LongRoPE, we continue pretraining the 512 token RoPE checkpoints for 2.5B tokens at each extended context length using position dependent rescaling.

\begin{figure}[ht]
    \centering
    \includegraphics[width=\linewidth]{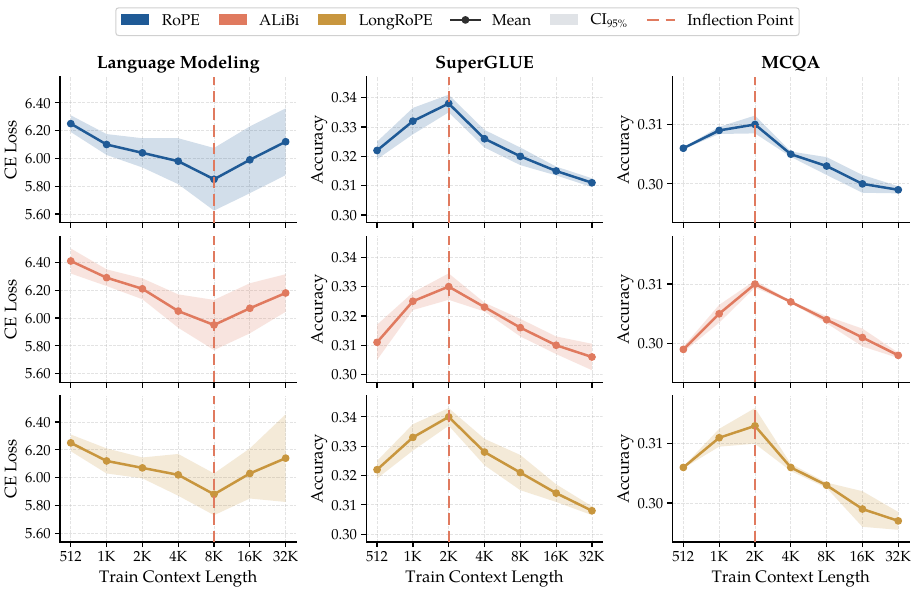}
    \caption{\textbf{The context length trend is robust to positional encoding choice.} Results for RoPE, ALiBi, and LongRoPE across language modeling, SuperGLUE, and closed-book MCQA. Each positional encoding exhibits the same qualitative pattern, where performance improves up to an intermediate context length and degrades as the training window increases further.}
    \label{fig:positional_encoding_ablation}
\end{figure}

As shown in \autoref{fig:positional_encoding_ablation}, all three positional encoding schemes exhibit the same qualitative trend. Language modeling performance improves up to an intermediate context length and subsequently degrades, while SuperGLUE and MCQA peak at shorter context lengths before declining. ALiBi yields somewhat weaker absolute performance than RoPE, whereas LongRoPE closely follows the RoPE results. Taken together, these experiments show that the observed context length trend is robust to the choice of positional encoding and is not an artifact of RoPE.

\clearpage

\section{Control Experiments}
\label{app:controls}
In this section, we conduct three control experiments. The first two test whether the effects in \S\ref{subsec_pt} and \S\ref{sec_solution_complexity} arise from informative train-time context rather than from length-related confounds. The third tests whether the no-context drop in \S\ref{sec_sft} arises from reduced parametric knowledge rather than from a change in input format.

\subsection{Stitched-Window Pretraining}
\label{app:stitched}

\paragraph{Setup.} We repeat the pretraining sweep in \S\ref{subsec_pt} at the 20M and 55M scales with \emph{stitched} windows, which are constructed from $\frac{W}{512}$ shuffled 512-token chunks drawn from different documents and are separated by a special token. The token budget, number of epochs, optimization steps, and exposure to short prefixes are matched to the coherent setting at every $W$. Therefore, stitched windows preserve the sequence length while removing informative long-range context.

\begin{figure}[h]
    \centering
    \includegraphics{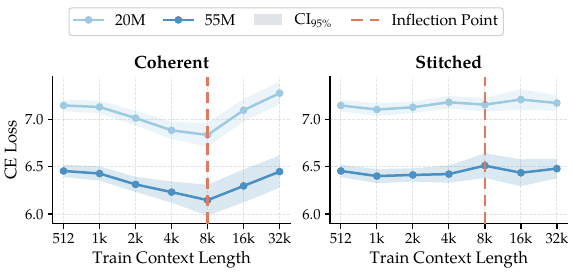}  
    \caption{\textbf{Both the gain and the degradation require informative distant context.} Language modeling cross-entropy for coherent (left) and stitched (right) pretraining windows at 20M and 55M. Shaded regions denote 95\% CI.}
    \label{fig:stitched}
\end{figure}

\paragraph{Results.} As shown in Figure~\ref{fig:stitched}, stitched windows remain flat across $W$, unlike the U-shaped curve of coherent windows. At every $W$, their loss stays within 0.07 of the $W = 512$ loss, which is well within seed variability. Since length-dependent optimization and data repetition are shared by both settings, they cannot explain this difference. Taken together, these results show that both the initial gain and the subsequent degradation require informative long-range context, consistent with the \IAP.

\subsection{Synthetic Evaluation With Fixed Demonstrations}
\label{app:fixed-demos}

\paragraph{Setup.} In \S\ref{sec_solution_complexity}, test-time context contains the same number of in-context demonstrations as the model observed during training, so test inputs differ across training conditions. To isolate the effect of train-time context, we re-evaluate all bitwise operation checkpoints with a fixed number of 128 test-time demonstrations, so that every model receives identical inputs.

\begin{figure}[h]
    \centering
    \includegraphics{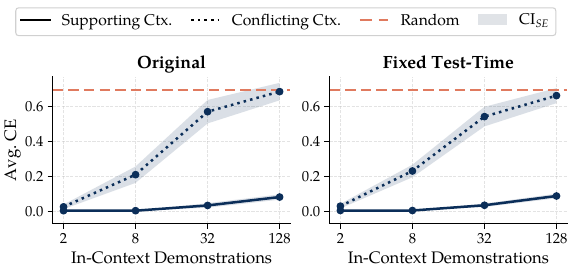}  
    \caption{\textbf{Context addiction persists under fixed test-time demonstrations.} Final-query CE loss for bitwise operations under supporting and conflicting context. Error bars denote standard error.}
\label{fig:fixed-demos}
\end{figure}

\paragraph{Results.} Figure~\ref{fig:fixed-demos} shows that the results remain nearly unchanged under fixed test-time demonstrations. The supporting--conflicting gap still grows with the number of train-time demonstrations. This indicates that context addiction in \autoref{fig:synth_grad_norm} reflects a strategy learned during training, rather than a difference in the evidence available at inference time.

\subsection{Supervised Fine-Tuning Evaluation in the Training Format}
\label{app:format-matched}

\paragraph{Setup.} In \S\ref{sec_sft}, every model is trained with eight prepended documents, so no-context evaluation removes the documents for all $k$. A model trained with $k=8$ could therefore learn a $k$-dependent access behavior, reading the answer off target-domain documents whenever eight documents are present. Its no-context drop could then reflect a change in input format rather than reduced parametric knowledge. To test this, we evaluate the $k=0$, $4$, and $8$ models on the held-out questions with eight irrelevant documents prepended, drawn from the paired domain as in App.~\ref{app_sft_doc_gen}. This matches the $k=0$ training format exactly. Since these documents contain no evidence regarding the answer, the model must rely on its parameters while the input format matches training. Under a format-dependent access behavior, $k=8$ accuracy should recover to the $k=0$ level.

\paragraph{Results.} \autoref{tab:format_matched} shows no such recovery. At every model size, $k=8$ accuracy stays within 0.5 points of its no-context accuracy, so matching the format recovers essentially none of the gap. Meanwhile, $k=0$ improves slightly in its own training format, by 0.5--1.7 points. As a result, accuracy decreases with $k$ at all five sizes, by 3.7--7.5 points from $k=0$ to $k=8$. Notably, irrelevant documents are also part of the $k=4$ training distribution. Yet $k=4$ still falls 3.1 points below $k=0$ on average, so the decline cannot be explained by $k=8$ models never having seen irrelevant context. All in all, these results favor reduced internalization over a format-dependent access behavior.

\clearpage

\section{Training Details}

\subsection{Natural Language Pretraining}
\label{sec_nl_pretrain}

This section provides training details for the natural language pretraining experiments in \S\ref{subsec_pt}. All models follow a Llama-2 architecture~\citep{llama2} with SwiGLU activations~\citep{swiglu}, RoPE positional embeddings~\citep{rope}, RMSNorm~\citep{rmsnorm}, and a shared vocabulary of 32000 tokens (Llama-2 tokenizer). The 259M and 750M models use grouped-query attention~\citep{gqa}, while the 20M and 55M models use multi-head attention~\citep{mha}. All models have untied input and output embeddings.

\begin{wraptable}[15]{r}{0.5\textwidth}
\vspace{-12pt}
\centering
\scriptsize
\setlength{\tabcolsep}{3pt}
\captionsetup{width=0.45\textwidth}

\begin{tabular}{@{}lrrrr@{}}
\toprule
& \textbf{20M} & \textbf{55M} & \textbf{259M} & \textbf{750M} \\
\midrule
\multicolumn{5}{@{}l}{\textit{Architecture}} \\
Hidden size          & 256   & 512   & 1024  & 1536 \\
Intermediate size    & 896   & 1792  & 3072  & 4608 \\
Layers               & 4     & 6     & 16    & 24   \\
Attn.\ heads (Q)     & 4     & 8     & 16    & 24   \\
Attn.\ heads (KV)    & 4     & 8     & 4     & 6    \\
Tied embeddings      & $\times$ & $\times$ & $\times$ & $\times$ \\
\midrule
\multicolumn{5}{@{}l}{\textit{Optimization}} \\
$\beta_1, \beta_2$   & \multicolumn{4}{r@{}}{0.9,\ 0.95} \\
$\epsilon$           & \multicolumn{4}{r@{}}{$10^{-8}$} \\
Weight decay         & \multicolumn{4}{r@{}}{0.1} \\
Peak LR              & $6 \times 10^{-4}$ & $4 \times 10^{-4}$ & $4 \times 10^{-4}$ & $3 \times 10^{-4}$ \\
Min LR               & $6 \times 10^{-5}$ & $4 \times 10^{-5}$ & $4 \times 10^{-5}$ & $3 \times 10^{-5}$ \\
\bottomrule
\end{tabular}
\caption{Hyperparameters per model scale.}
\label{tab:model_arch}
\end{wraptable}

All models are trained for 9537 optimization steps on 10B tokens (4 epochs of a 2.5B-token corpus), with approximately 1.05M tokens per step. Optimization uses AdamW~\citep{adamw} with gradient clipping at 1.0, a linear warmup of 2000 steps, and cosine decay thereafter. All models are trained in \texttt{bfloat16} with FlashAttention-2~\citep{flashattention2}. The 20M, 55M, and 259M models use data parallelism across two NVIDIA A100 80GB GPUs, whereas the 750M models are trained on four NVIDIA H100 80GB GPUs. Remaining hyperparameters are provided in \autoref{tab:model_arch}. For each configuration, we run multiple random seeds. We use five seeds
for the 20M model and three seeds for the 55M, 259M, and 750M models. The seeds
shared across all model scales are 42, 2026, and 1000. For the 20M model,
we additionally use seeds 9999 and 12151. All reported evaluation results
are averaged across seeds. Shaded standard error bands denote the standard
error computed across seed-level means. 

\begin{wraptable}[11]{r}{0.56\textwidth}
\centering
\vspace{-12pt}
\scriptsize
\setlength{\tabcolsep}{2.5pt}

\resizebox{0.54\textwidth}{!}{%
\begin{tabular}{lrrrrrrr}
\toprule
 & \multicolumn{7}{c}{\textbf{Context Length (tokens)}} \\
\cmidrule(lr){2-8}
\textbf{Model} & 512 & 1K & 2K & 4K & 8K & 16K & 32K \\
\midrule
20M  & 4.77  & 4.74  & 4.91  & 5.12  & 5.52  & 6.40  & 8.17   \\
55M  & 8.23  & 8.34  & 8.49  & 9.19  & 10.53 & 13.22 & 18.54  \\
259M & 27.90 & 28.66 & 30.13 & 33.57 & 40.72 & 55.04 & 83.44  \\
750M & 35.87 & 37.26 & 35.81 & 40.85 & 53.72 & 73.36 & 118.13 \\
\bottomrule
\end{tabular}
}
\caption{Average pretraining cost in GPU hours across model sizes and context lengths.}
\label{tab:avg_gpu_hours}
\end{wraptable}

\autoref{tab:avg_gpu_hours} reports the average pretraining cost in
GPU hours for each model scale and context length. Training cost is
relatively stable at short context lengths but grows noticeably at
longer contexts, particularly beyond 8K tokens. For instance, the 259M
model requires 27.90 GPU hours at 512 tokens and 83.44 GPU hours at 32K
tokens. Other models show the same qualitative trend. Because the 20M, 55M, and 259M models are trained on A100 GPUs whereas the 750M model is trained on H100 GPUs, absolute GPU hours should not be compared directly across model scales. These measurements make explicit the compute tradeoff associated with extending the context length during pretraining.

\subsection{Supervised Fine-Tuning}
\label{app_sft_details}

This section provides training details for the SFT experiments in \S\ref{sec_sft}. We perform supervised fine-tuning using the \texttt{verl}
infrastructure~\citep{verl}. For each model and domain, we fine-tune for 5 epochs with
a learning rate of \(1\times10^{-4}\), cosine learning-rate decay, and
gradient clipping at 1.0. We use a maximum sequence length of 1024 tokens
and apply right truncation to examples exceeding this length.

We use LoRA with rank $r=64$ and scaling parameter \(\alpha=128\). Training
is performed in \texttt{bfloat16} precision with gradient checkpointing.

\subsection{Synthetic Pretraining}
\label{app_synth_pretrain}
This section provides training details for the synthetic pretraining experiments in \S\ref{sec_solution_complexity}.

\paragraph{Task format.}
Each example consists of a sequence of input-output demonstrations followed by a
final query. Models are trained to predict only the output tokens of the final
query. For a training context length $k$, the prompt contains $k$ demonstrations
sampled from the same task family, followed by one held-out query from that task.

A generic prompt has the form $(x_1, y_1), \ldots, (x_k, y_k), x_\star \mapsto y_\star,$
where $(x_i,y_i)$ are in-context demonstrations and $(x_\star,y_\star)$ is the
final query. The supervised loss is computed only on $y_\star$. This setup makes
the context useful for solving the final query while allowing us to vary the
amount of useful context independently of the task family.

\paragraph{Task families.}
We use four deterministic task families.

\begin{enumerate}[leftmargin=*, nosep]
    \item \textbf{Unary bitwise operations.}
    Inputs are 16-bit binary strings. Each task applies a unary bitwise
    operation to the input string, such as \textsc{NOT}, which maps each bit
    to its complement.\footnote{The full operation list is
    \textsc{IDENTITY}, \textsc{NOT}, \textsc{REVERSE},
    \textsc{REVERSE\_NOT}, \textsc{POPCOUNT}, \textsc{PARITY},
    \textsc{LEADING\_ZEROS}, \textsc{TRAILING\_ZEROS}, \textsc{MSB},
    \textsc{LSB}, \textsc{UPPER\_HALF}, and \textsc{LOWER\_HALF}.}
    The valid output tokens are \texttt{0} and \texttt{1}. We train each model for 25 epochs.
    
    \item \textbf{String transformations.}
    Inputs are 8-letter strings over a fixed alphabet. Each task applies a
    deterministic string transformation, such as \textsc{REVERSE}.\footnote{The
    full operation list is \textsc{IDENTITY}, \textsc{REVERSE},
    \textsc{ROTATE\_LEFT\_1}, \textsc{ROTATE\_RIGHT\_1},
    \textsc{ROTATE\_LEFT\_2}, \textsc{ROTATE\_RIGHT\_2},
    \textsc{SWAP\_HALVES}, \textsc{REVERSE\_FIRST\_HALF},
    \textsc{REVERSE\_SECOND\_HALF}, \textsc{REVERSE\_EACH\_2BLOCK},
    \textsc{REVERSE\_EACH\_4BLOCK}, and \textsc{INTERLEAVE\_HALVES}.}
    The output is another string over the same alphabet. We train each model for 25 epochs.

    \item \textbf{Digit-wise mod10 arithmetic.}
    Inputs are 5-digit strings. Each task applies a digit-wise arithmetic
    operation modulo 10. For example, a task may add a fixed digit-wise offset to
    each input digit, with all arithmetic performed modulo 10. We train each model for 10 epochs.

    \item \textbf{Caesar cipher.}
    Inputs are 5-digit strings. Each task shifts every digit by a global
    offset modulo 10. Importantly, the same offset is
    applied to all positions. We train each model for 10 epochs.
\end{enumerate}

\paragraph{Data generation.}
For each task family, datasets are generated deterministically from the
corresponding input-output rule. We generate approximately 33K training examples
and 10K test examples per task. Training and test inputs are disjoint. Each task
uses a task-specific tokenizer with fewer than 30 tokens, including symbols for
digits or letters, operation delimiters, separators, and special tokens.
For each context length $k$, we construct examples by sampling $k$
demonstrations and one final query from the same task. 
The final query does not appear among the demonstrations. The training set size
is held fixed across context lengths so that changes in performance reflect the
amount of useful context available per example rather than the number of
optimization examples.

\begin{wraptable}[16]{r}{0.5\textwidth}
\vspace{-12pt}
\centering
\scriptsize
\setlength{\tabcolsep}{3pt}
\captionsetup{width=\linewidth}

\resizebox{\linewidth}{!}{%
\begin{tabular}{@{}lccc@{}}
\toprule
& \textbf{0.3M} & \textbf{1.5M} & \textbf{7.5M} \\
\midrule
\multicolumn{4}{@{}l}{\textit{Architecture}} \\
Layers              & 4 & 6 & 8 \\
Hidden size         & 64 & 128 & 256 \\
Attn.\ heads (Q)    & 2 & 4 & 8 \\
Attn.\ heads (KV)   & 1 & 2 & 4 \\
Intermediate size   & 256 & 512 & 1024 \\
\midrule
\multicolumn{4}{@{}l}{\textit{Optimization}} \\
Optimizer           & \multicolumn{3}{c@{}}{AdamW} \\
Peak LR             & $1.0{\times}10^{-4}$ & $7.0{\times}10^{-5}$ & $5.0{\times}10^{-5}$ \\
Min LR              & \multicolumn{3}{c@{}}{$1.0{\times}10^{-5}$} \\
Batch size          & 32 & 64 & 128 \\
Training steps      & 25{,}000 & 12{,}500 & 6{,}250 \\
Warmup steps        & 1{,}250 & 625 & 312 \\
Weight decay        & \multicolumn{3}{c@{}}{0.001} \\
Gradient clipping   & \multicolumn{3}{c@{}}{1.0} \\
\bottomrule
\end{tabular}
}
\caption{Synthetic pretraining model configurations and hyperparameters.}
\label{tab:synth_model_configs}
\end{wraptable}

\paragraph{Architecture and optimization.}
We train decoder-only transformer language models at three scales: 0.3M, 1.5M,
and 7.5M parameters. All models are trained from scratch on each task family
separately. The architecture follows the same causal language modeling setup as
the natural language pretraining experiments in \S\ref{subsec_pt}, but uses
smaller widths, depths, and task-specific vocabularies.
All models are trained with a causal attention mask and next token prediction
objective, but the loss is applied only to the final-query output tokens. We vary
the number of in-context demonstrations available during training and hold the
task family, data-generating process, optimizer, and model scale fixed.

\paragraph{Evaluation conditions.}
We evaluate each trained model under two context conditions.

\begin{enumerate}[leftmargin=*, nosep]
    \item \textbf{Supporting context.}
    The context contains demonstrations generated by the correct task rule. This
    condition measures performance when the model can rely on useful in-context
    evidence.

    \item \textbf{Conflicting context.}
    The context contains demonstrations generated by a consistent but incorrect
    rule from the same task family. This condition measures whether the model follows the
    supplied context even when it conflicts with the parametrically correct rule.
\end{enumerate}

We report the gap between supporting-context and conflicting-context performance
as a measure of context addiction. A larger gap indicates that the model is more
sensitive to the correctness of the provided context and therefore relies less
robustly on a parametrically internalized task rule. 
We omit no-context evaluation in this synthetic setting because models are
trained only on final-query losses following in-context demonstrations. Removing
the demonstrations at test time changes the input format and sequence
distribution, making no-context prompts out-of-distribution for these models.

\clearpage

\section{Evaluation Benchmarks}
\label{app_benchmarks}
\subsection{Language Modeling Suite}
\label{app_lm_suite}

The language modeling suite consists of LAMBADA, Penn Treebank (PTB), and
WikiSPAN. LAMBADA evaluates word prediction in passages where the target word
depends on broad discourse context rather than only local syntax
\citep{paperno2016lambada}. PTB provides a standard corpus-level benchmark for
measuring next token prediction on natural text~\citep{ptb}.
WikiSPAN evaluates language modeling over time-indexed Wikipedia documents,
providing an additional testbed for factual and distributional variation in
naturally occurring text~\citep{wikispan}. These benchmarks are well
suited to our evaluation because they directly measure the pretraining
objective. They therefore test whether increasing the train-time
context window improves objective-aligned compression or degrades general
next token prediction beyond an intermediate optimum.

\subsection{SuperGLUE}
\label{app_superglue}

SuperGLUE is a suite of natural language understanding tasks that
measure capabilities such as entailment, coreference, causal reasoning, and
word-sense disambiguation~\citep{wang2020superglue}. The suite aggregates several benchmarks, including BoolQ~\citep{clark2019boolq}, CommitmentBank~\citep{jiang-de-marneffe-2019-evaluating}, COPA~\citep{roemmele2011choice}, MultiRC~\citep{khashabi2018looking}, ReCoRD~\citep{zhang2018recordbridginggaphuman}, RTE~\citep{rte}, WiC~\citep{pilehvar2019wicwordincontextdatasetevaluating}, and WSC~\citep{sakaguchi2019winograndeadversarialwinogradschema}. We
include SuperGLUE as an intermediate evaluation between language modeling and
closed-book multiple-choice question answering. Unlike the language modeling
suite, SuperGLUE probes whether representations learned during pretraining
transfer to structured language understanding tasks. Unlike the MCQA suite, it
does not primarily target factual recall. This makes it well suited for testing
whether train-time context length affects general linguistic and reasoning
competence rather than only the model's ability to store task-specific
knowledge.

\subsection{Closed-Book MCQA Suite}
\label{app_mcqa}

The closed-book MCQA suite consists of ARC-Easy and ARC-Challenge
\citep{clark2018think}, CommonsenseQA~\citep{talmor2019commonsense},
HellaSwag~\citep{zellers2019hellaswag}, OpenBookQA
\citep{mihaylov2018suit}, PIQA~\citep{bisk2019piqa}, SocialIQA
\citep{sap2019socialiqa}, MedQA~\citep{jin2020disease}, TruthfulQA
\citep{lin2022truthfulqa}, SciQ~\citep{welbl2017crowdsourcing}, and
WinoGrande~\citep{sakaguchi2020winogrande}. These benchmarks cover
complementary forms of knowledge and reasoning, including grade-school science,
general commonsense, physical commonsense, social commonsense, medical
knowledge, truthfulness, and commonsense coreference. We evaluate them in a
closed-book setting, using only the question and answer choices without
auxiliary retrieval or supporting documents. This protocol directly supports our
goal of measuring parametric knowledge. If longer-context training shifts
predictive structure from weights toward context-conditioned computation, the
effect should be visible when the model must answer without context
at test time.

\clearpage

\section{Per-Dataset Evaluation Results}
\label{sec_dataset_decomp}

We report evaluation results for each individual dataset underlying the aggregate results in \autoref{fig:pt_res}. This per-dataset analysis tests whether the observed \emph{inverted-U} pattern is driven by a small number of benchmarks or is visible across the underlying tasks.

\begin{figure}[ht]
    \centering
    \includegraphics[width=\linewidth]{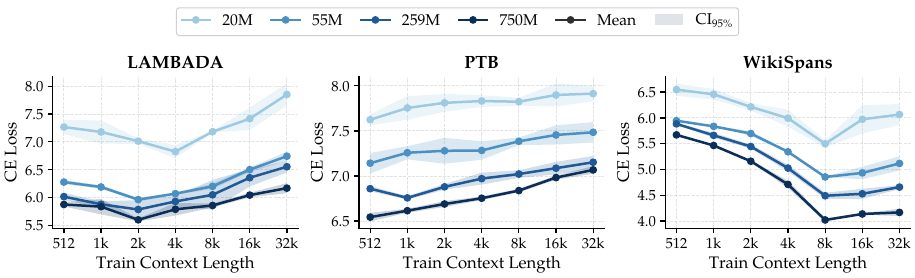}
    \caption{\textbf{Language modeling results by dataset.} Across LAMBADA, PTB, and WikiSPAN, longer pretraining windows initially improve next token prediction but eventually degrade performance at longer windows. The dataset-level trends mirror the aggregate language modeling curve in \autoref{fig:pt_res}, with the strongest performance generally occurring at intermediate context lengths.}
    \label{fig:pt_res_per_dataset_lm}
\end{figure}

\begin{wraptable}[12]{r}{0.6\textwidth}
\vspace{-12pt}
\centering
\small
\setlength{\tabcolsep}{4pt}
\renewcommand{\arraystretch}{1.10}

\begin{tabular}{lccccc|c}
\toprule
Dataset & $p_{05}$ & $p_{25}$ & Mean & $p_{75}$ & $p_{95}$ & Inflection Point $W$ \\
\midrule
PTB       & 7   & 18  & 28  & 37  & 54  & 512  \\
LAMBADA   & 71  & 78  & 88  & 95  & 113 & 2048 \\
WikiSPAN  & 216 & 311 & 504 & 668 & 793 & 8192 \\
\bottomrule
\end{tabular}
\caption{\textbf{Longer language modeling examples favor longer pretraining context windows.}
Token-length statistics for each language modeling dataset and the training window $W$ yielding the lowest evaluation loss. Token counts are measured with the \texttt{Llama-2-7b} tokenizer.}
\label{tab:lm_length_best_window}
\end{wraptable}
This decomposition provides suggestive evidence that the preferred train-time context length partly tracks the effective length of the evaluation distribution. As shown in \autoref{tab:lm_length_best_window}, PTB contains the shortest examples and attains its lowest loss at $W = 512$. LAMBADA is longer and attains its lowest loss at $W=2048$. WikiSPAN contains the longest examples and attains its lowest loss at $W=8192$. This trend is consistent with the hypothesis that shorter evaluation tasks saturate at shorter training windows, while longer language modeling contexts benefit from longer train-time context before the long-context degradation appears. The downstream results in \autoref{tab:benchmark_token_statistics} follow the same broad pattern: MCQA and SuperGLUE have mean lengths of $65$ and $135$ tokens, respectively, and both exhibit an inflection point at $W=2048$. Across all five benchmarks, the inflection point therefore falls between 15 and 32 times the mean example length.

\clearpage
\begin{wraptable}[10]{r}{0.6\textwidth}
\centering
\small
\setlength{\tabcolsep}{4pt}
\renewcommand{\arraystretch}{1.10}

\begin{tabular}{lccccc|c}
\toprule
Dataset & $p_{05}$ & $p_{25}$ & Mean & $p_{75}$ & $p_{95}$ & Inflection Point $W$ \\
\midrule
MCQA      & 35  & 48  & 65  & 77  & 110 & 2048 \\
SuperGLUE & 72  & 101 & 135 & 160 & 221 & 2048 \\
\bottomrule
\end{tabular}
\caption{\textbf{The trend observed for language modeling benchmarks continues on MCQA and SuperGLUE.}
Token-length statistics for each benchmark and the training window $W$ at which performance begins to degrade.}
\label{tab:benchmark_token_statistics}
\end{wraptable}
For example, $512/28\approx18$, $2048/88\approx23$, $8192/504\approx16$, $2048/65\approx32$, and $2048/135\approx15$. Although this relationship is only approximate and the selected context windows are discretized, it suggests a simple scaling rule, where performance tends to saturate, or begin to degrade, once the train-time context window reaches 15--32 times the mean evaluation length.
\begin{figure}[ht]
    \centering
    \includegraphics[width=\linewidth]{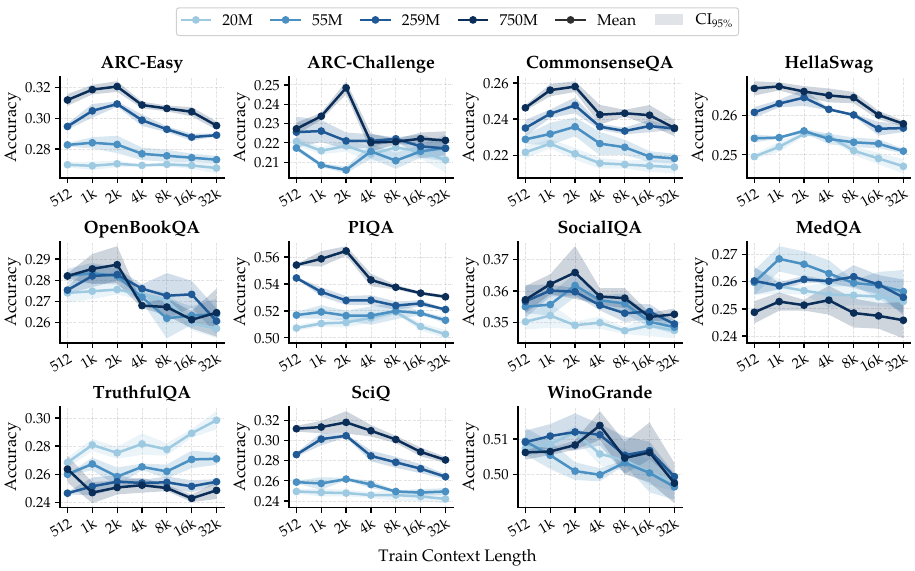}
    \caption{\textbf{Closed-book MCQA results by dataset.} Per-benchmark multiple-choice accuracy shows the same qualitative pattern as the aggregate MCQA result, where performance improves up to intermediate context windows and declines for longer ones.}
    \label{fig:pt_res_per_dataset_mcqa}
\end{figure}

Overall, the dataset-level decompositions (\autoref{fig:pt_res_per_dataset_lm} and \autoref{fig:pt_res_per_dataset_mcqa}) support the conclusion that context
length acts as a systematic training variable. While individual benchmarks vary
in their sensitivity to window size, the broad trend remains consistent:
increasing train-time context helps only up to an intermediate regime, after
which longer windows reduce both objective-aligned language modeling performance
and closed-book downstream accuracy.

\clearpage

\section{Statistical Significance Tests}
\label{app_significance_tests}

This section reports statistical significance tests for the main empirical results in
Figs.~\ref{fig:pt_res}, \ref{fig:phi_results}, \ref{fig:sft_res},
\ref{fig:synth_grad_norm}, \ref{fig:grad_all_combined},
\ref{fig:qwen3_ffn_attn_k8_avg}, and \ref{fig:attn_mass_qwen1}.
Throughout, we use a significance level of $\alpha=0.05$.

\subsection{Long and Short Context Phi-3 and OLMo 3 Variants (\autoref{fig:phi_results})}
\label{app_sig_long_context_models}

We compare each long-context model with its corresponding short-context variant
using a one-sided two-proportion score test in the hypothesized direction that
the long-context variant performs worse. For Phi-3, we compare the 128K variants
against their corresponding 4K variants; for OLMo 3, we compare the 65K variants
against their corresponding 8K variants. We apply Holm correction within each
model-family comparison set.

\begin{table}[h]
\centering
\small
\setlength{\tabcolsep}{3.5pt}

{%
\begin{tabular}{llccc}
\toprule
\textbf{Model} &
\textbf{Evaluation} &
\textbf{MMLU $p$} &
\textbf{BBH $p$} &
\textbf{MCQA $p$} \\
\midrule

\emph{Phi-3} \\

3.8B & Few-shot
& 0.0048
& 0.1920
& $1.23{\times}10^{-4}$ \\

3.8B & Zero-shot
& $4.39{\times}10^{-9}$
& $2.19{\times}10^{-14}$
& $5.86{\times}10^{-7}$ \\

14B & Few-shot
& 0.0028
& 0.0155
& $1.72{\times}10^{-5}$ \\

14B & Zero-shot
& $2.91{\times}10^{-39}$
& $2.16{\times}10^{-11}$
& $4.39{\times}10^{-32}$ \\

\midrule

\emph{OLMo 3} \\

7B & Few-shot
& 0.1204
& $6.83{\times}10^{-4}$
& 0.0189 \\

7B & Zero-shot
& 0.0071
& $4.79{\times}10^{-22}$
& $4.66{\times}10^{-5}$ \\

32B & Few-shot
& 0.0305
& $4.99{\times}10^{-80}$
& 0.2669 \\

32B & Zero-shot
& ---
& $1.57{\times}10^{-4}$
& $6.21{\times}10^{-3}$ \\

\bottomrule
\end{tabular}%
}
\caption{\textbf{Significance tests for the comparisons in \autoref{fig:phi_results}.} We report $p$-values from one-sided two-proportion score tests comparing Phi-3 128K against 4K variants and OLMo 3 65K against 8K variants.}
\label{tab:sig_long_context}
\end{table}

\paragraph{Phi-3.}
The 128K variants perform significantly worse than their corresponding 4K
variants in 11 of the 12 comparisons.

\paragraph{OLMo 3.}
The 65K variants perform significantly worse than their corresponding 8K
variants in 9 of the 11 comparisons. We do not conduct a test for 32B zero-shot MMLU comparison, since OLMo 3 65K performs better.

Taken together, the long-context variants are significantly worse in
20 of the 23 reported Phi-3 and OLMo 3 comparisons, providing statistical
support for the motivating pattern in \autoref{fig:phi_results}.

\subsection{Pretraining Context Length Inflection Tests (\autoref{fig:pt_res})}
\label{app_sig_pretraining}

\begin{table}[h]
\centering
\small
\setlength{\tabcolsep}{4pt}

{%
\begin{tabular}{lccc}
\toprule
\textbf{Model} & \textbf{Language Modeling} &
\textbf{SuperGLUE} & \textbf{MCQA} \\
\midrule
20M  & 0.016 & 0.041 & 0.022 \\
55M  & 0.032 & 0.012 & 0.013 \\
259M & 0.043 & 0.003 & 0.005 \\
750M & 0.021 & 0.004 & 0.004 \\
\bottomrule
\end{tabular}%
}
\caption{\textbf{Segmented regression inflection tests for \autoref{fig:pt_res}.} We report $p$-values for the corresponding inflection test.}
\label{tab:sig_pretraining_inflection}
\end{table}

We apply segmented regression tests to assess whether performance exhibits a statistically significant inflection as training context length increases. Degradation begins after 8K tokens for language modeling and after 2K tokens for SuperGLUE and MCQA, consistent with the intermediate optima observed in \autoref{fig:pt_res}. The inflection is statistically significant for all three evaluation suites at every model scale, indicating that the post-optimum degradation persists as model capacity increases.

\subsection{Ordered SFT Trends (\autoref{fig:sft_res})}
\label{app_sig_sft_trends}

We use blocked ordered trend tests over train-time $k\in\{0,4,8\}$, treating
the four domains as equally weighted repeated blocks. For no-context accuracy,
the one-sided alternative is a decreasing trend with $k$; for the
supporting--conflicting accuracy gap, the one-sided alternative is an increasing trend. The ordered effect is significant in 9 of 10 comparisons, as no-context accuracy decreases significantly for every model except Qwen3-8B, while the
supporting--conflicting gap increases significantly for all five model sizes. These results show that the qualitative trends in \autoref{fig:sft_res} are consistent across domains.

\begin{table}[h]
\centering
\small
\setlength{\tabcolsep}{4pt}

{%
\begin{tabular}{lrrrr}
\toprule
\textbf{Model} & \multicolumn{2}{c}{\textbf{No Ctx.}} &
\multicolumn{2}{c}{\textbf{Supporting Ctx. -- Conflicting Ctx.}} \\
\cmidrule(lr){2-3}\cmidrule(lr){4-5}
& \textbf{$\Delta$} & \textbf{$p$} &
\textbf{$\Delta$} & \textbf{$p$} \\
\midrule
0.6B & $-5.4$ & 0.0074 & $+42.8$ & 0.0067 \\
1.7B & $-3.6$ & 0.0068 & $+25.6$ & 0.0084 \\
4B   & $-5.3$ & 0.0097 & $+23.3$ & 0.0087 \\
8B   & $-1.7$ & 0.0645 & $+23.6$ & 0.0074 \\
14B  & $-6.3$ & 0.0062 & $+18.0$ & 0.0096 \\
\bottomrule
\end{tabular}%
}
\caption{\textbf{Ordered trend tests for \autoref{fig:sft_res}.} $\Delta$ is the change from $k=0$ to $k=8$ for the corresponding evaluation quantity; $p$ is the one-sided ordered trend test $p$-value.}
\label{tab:sig_sft_trends}
\end{table}

\subsection{Gradient Norm Trends in the Synthetic Study (\autoref{fig:synth_grad_norm})}
\label{app_sig_synth_grad}

We apply a one-sided Kendall rank trend test to the mean gradient norm of each
independent training run, testing whether gradient norm decreases as the number
of in-context demonstrations increases. Gradient norms decrease significantly
for bitwise operations and string operations, whereas mod10 and Caesar exhibit
nonsignificant trends. Thus, significantly decreasing gradient norm is observed
specifically for the tasks that exhibit context addiction.

\begin{table}[h]
\centering
\small
\setlength{\tabcolsep}{5pt}
 
\begin{tabular}{lrr}
\toprule
\textbf{Task} & \textbf{Kendall $\tau$} & \textbf{$p$} \\
\midrule
Bitwise Ops. & $-0.409$ & 0.0037 \\
String Ops.  & $-0.641$ & $6.26{\times}10^{-10}$ \\
Mod10        & $+0.260$ & 0.9912 \\
Caesar       & $-0.159$ & 0.0704 \\
\bottomrule
\end{tabular}
\caption{\textbf{Gradient norm trend tests for \autoref{fig:synth_grad_norm}.}}
\label{tab:sig_synth_grad}
\end{table}

\clearpage

\subsection{Module-Level Gradient Allocation (\autoref{fig:grad_all_combined})}
\label{app_sig_grad_all}

\textbf{Supervised fine-tuning (\autoref{fig:grad_all}).}
At each model-domain pair, we match FFN-to-SA gradient norm ratios by training
step and test the mean paired difference between $k=8$ and no-context training
using a two-sided normal test with Newey--West heteroskedasticity- and
autocorrelation-consistent (HAC) standard errors. The $k=8$ trajectory has a
lower mean ratio in all 20 comparisons, and 19 of 20 comparisons are
significant.
\begin{table}[h]
\centering
\small
\setlength{\tabcolsep}{3.5pt}

\resizebox{0.7\linewidth}{!}{%
\begin{tabular}{lcccc}
\toprule
\textbf{Model} & \textbf{Economics} & \textbf{Law} &
\textbf{Health} & \textbf{Psychology} \\
\midrule
0.6B & $1.42{\times}10^{-51}$ & 0.0706 &
0.0014 & $3.56{\times}10^{-12}$ \\
1.7B & 0.0195 & $3.60{\times}10^{-30}$ &
0.0155 & $4.57{\times}10^{-4}$ \\
4B   & 0.0087 & $3.27{\times}10^{-6}$ &
0.0024 & $1.04{\times}10^{-6}$ \\
8B   & $4.50{\times}10^{-7}$ & $9.13{\times}10^{-19}$ &
$1.63{\times}10^{-6}$ & $5.95{\times}10^{-12}$ \\
14B  & 0.0040 & $4.52{\times}10^{-8}$ &
0.0002 & 0.0093 \\
\bottomrule
\end{tabular}%
}
\caption{\textbf{SFT gradient allocation tests for
\autoref{fig:grad_all}.} We report two-sided HAC-normal test $p$-values
for the paired difference in FFN-to-SA gradient norm ratio between $k=8$ and
no-context training.}
\label{tab:sig_grad_all_sft}
\end{table}

\textbf{Pretraining (\autoref{fig:grad_all_pt}).}
We compute the mean change in the FFN-to-SA gradient norm ratio per doubling
of the training context length and test whether this slope is negative. The
ratio decreases significantly with increasing context length at every model
scale. This consistent negative trend indicates that the shift in gradient pressure toward self-attention persists as model capacity increases.

\begin{table}[h]
\centering
\small
\setlength{\tabcolsep}{5pt}
{%
\begin{tabular}{lcc}
\toprule
\textbf{Model} & \textbf{Mean Slope per Context Doubling} & \textbf{$p$} \\
\midrule
20M  & $-0.038$ & 0.0013 \\
55M  & $-0.047$ & 0.0051 \\
259M & $-0.054$ & 0.0016 \\
750M & $-0.045$ & 0.0074 \\
\bottomrule
\end{tabular}%
}
\caption{\textbf{Pretraining gradient allocation trend tests for
\autoref{fig:grad_all_pt}.} The slope is the mean change in FFN-to-SA
gradient norm ratio per context length doubling.}
\label{tab:sig_grad_all_pt}
\end{table}

\subsection{Module-Restricted Supervised Fine-Tuning (\autoref{fig:qwen3_ffn_attn_k8_avg})}
\label{app_sig_module_restricted}

\begin{table}[h]
\centering
\small
\setlength{\tabcolsep}{3.5pt}

{%
\begin{tabular}{lrrrrrr}
\toprule
\textbf{Model} &
\multicolumn{2}{c}{\textbf{No Ctx.} ($\downarrow$)} &
\multicolumn{2}{c}{\textbf{Supporting Ctx.} ($\uparrow$)} &
\multicolumn{2}{c}{\textbf{Conflicting Ctx.} ($\downarrow$)} \\
\cmidrule(lr){2-3}
\cmidrule(lr){4-5}
\cmidrule(lr){6-7}
& $\Delta$ & $p$
& $\Delta$ & $p$
& $\Delta$ & $p$ \\
\midrule
0.6B
& $-1.9$ & 0.0733
& $+4.5$ & 0.0036
& $-5.9$ & 0.0005 \\

1.7B
& $-1.9$ & 0.0739
& $+9.0$ & 0.0001
& $-10.0$ & $3.17{\times}10^{-6}$ \\

4B
& $-2.9$ & 0.0015
& $+2.2$ & 0.0135
& $-12.0$ & $1.49{\times}10^{-7}$ \\
\bottomrule
\end{tabular}%
}
\caption{\textbf{Module-restricted fine-tuning tests for
\autoref{fig:qwen3_ffn_attn_k8_avg}.}
$\Delta$ denotes SA-only minus FFN-only accuracy, and $p$ is the
one-sided two-proportion score test $p$-value. Arrows indicate the
hypothesized direction for SA-only relative to FFN-only.}
\label{tab:sig_module_restricted}
\end{table}

For each model size and evaluation setting, we use a one-sided two-proportion
score test in the hypothesized direction, using the independence approximation.
The predicted direction holds in all nine comparisons, with seven of nine being
significant. SA-only tuning is significantly better than FFN-only tuning under
supporting context and significantly worse under conflicting context for all
three models. Under no-context evaluation, SA-only tuning is significantly
worse only for Qwen3-4B.

\subsection{Inference-Time Attention to Context (\autoref{fig:attn_mass_qwen1})}
\label{app_sig_attention}

For each domain, we use a one-sided cluster-based permutation test to determine
whether $k=8$ SFT increases attention to context tokens relative to no-context
SFT. For each question, the layer-wise difference curve is sign-flipped during
permutation, thereby preserving dependence across layers. Positive differences
across contiguous layers are grouped into clusters, and the maximum cluster
statistic controls for multiple comparisons over layers. The $k=8$ checkpoint
assigns significantly greater attention mass to context tokens in layers 4--16
across all four domains.

\begin{table}[h]
\centering
\small
\setlength{\tabcolsep}{5pt}

\begin{tabular}{lcr}
\toprule
\textbf{Domain} & \textbf{Layers} & \textbf{$p$} \\
\midrule
Law        & 4--16 & $1.4{\times}10^{-4}$ \\
Health     & 4--16 & 0.0040 \\
Psychology & 4--16 & $5.4{\times}10^{-4}$ \\
Economics  & 4--16 & 0.0180 \\
\bottomrule
\end{tabular}
\caption{\textbf{Permutation tests for \autoref{fig:attn_mass_qwen1}.}}
\label{tab:sig_attention}
\end{table}

\clearpage

\section{Context Construction for Supervised Fine-Tuning}
\label{app_sft_doc_gen}

We construct the supervised fine-tuning data from four MMLU-Pro domains, namely Health,
Economics, Law, and Psychology. Since MMLU-Pro provides only a test split, we
create an 80/20 split within each domain. The 80\% portion is used for
supervised fine-tuning, and the remaining 20\% portion is reserved for held-out
evaluation.

For each question, we generate two sets of documents. The first set is
\textbf{supporting context}, which is intended to support the correct answer. The
second set is \textbf{conflicting context}, which is intended to plausibly steer the
model toward an incorrect answer. We generate both sets using \texttt{gemini-3.1-flash-lite-preview}~\citep{googledeepmind2026gemini31flashlite} through the
Gemini API. We use the model's default generation hyperparameters, with
\texttt{temperature=1.0}, \texttt{topP=0.95}, \texttt{topK=64}, and
\texttt{candidateCount=1}. The maximum output length is set by the model limit
of 65536 tokens. The system prompt used for generation is provided below:

\begin{quote}
\small
\textbf{System prompt.}

You are an expert question analyst. Given a multiple choice question, you will
generate context and reasoning in JSON format only, with no extra text or
markdown.

Your output must be a JSON object with exactly these four keys.

\begin{itemize}[leftmargin=*]
    \item \texttt{supporting\_context}: Exactly 8 sentences of factual background
    information that directly supports arriving at the correct answer. Each
    sentence must be fully self-contained and independent. No sentence should
    reference, depend on, or follow logically from any other sentence. Avoid
    discourse connectives like ``furthermore'', ``however'', ``therefore'', and
    ``this means''.

    \item \texttt{conflicting\_context}: Exactly 8 sentences that sound plausible and
    relevant but subtly steer reasoning toward the second most likely incorrect
    answer, without explicitly mentioning any answer choice. Each sentence must
    be fully self-contained and independent. No sentence should reference,
    depend on, or follow logically from any other sentence. Avoid discourse
    connectives.

    \item \texttt{trick\_answer}: The single answer option letter, such as
    \texttt{A}, \texttt{B}, or \texttt{C}, that the \texttt{conflicting\_context} is
    designed to steer toward.

    \item \texttt{reasoning}: A concise 1 to 3 sentence explanation of why the
    correct answer is correct.
\end{itemize}
\end{quote}

\begin{table*}[bh]
\centering
\small
\setlength{\tabcolsep}{5pt}
\renewcommand{\arraystretch}{1.15}

\begin{tcolorbox}[
  width=\textwidth,
  colback=black!2,
  colframe=black!45,
  boxrule=0.45pt,
  arc=2pt,
  left=6pt,
  right=6pt,
  top=4pt,
  bottom=4pt,
  before skip=0pt,
  after skip=5pt
]
\textbf{Question:}
Which of the following is most likely to produce symptoms similar to anxiety?

\textbf{Options:}
\textcolor{correctgreen}{\textbf{A) Hyperthyroidism}}
\qquad
\textcolor{incorrectred}{\textbf{B) Addison's disease}}
\end{tcolorbox}

\begin{tabularx}{\textwidth}{
  >{\centering\arraybackslash}p{0.8cm}
  |Y|Y
}
\toprule
&
\cellcolor{targetbg}
\centering\textbf{Target Domain (Psychology)}
&
\cellcolor{otherbg}
\centering\textbf{Other Domain (e.g. Economics)}
\tabularnewline
\midrule

\cellcolor{supportbg}
\vcell{%
  \centering
  \rotatebox[origin=c]{90}{\textbf{Supporting}}%
}
&
\vcell{%
  \textit{Hyperthyroidism involves an overactive thyroid gland that
  produces an excess of thyroid hormones.}

  \par\smallskip
  \trainlabel{} $+$ \testlabel
}
&
\vcell{%
  \textit{Price leadership models characterize industries where firms
  adopt the pricing strategy set by a dominant entity.}

  \par\smallskip
  \trainlabel
}
\tabularnewline[-\rowheight]

\cellcolor{supportbg}
\printcellmiddle
&
\printcelltop
&
\printcelltop
\tabularnewline
\midrule

\cellcolor{conflictbg}
\vcell{%
  \centering
  \rotatebox[origin=c]{90}{\textbf{Conflicting}}%
}
&
\vcell{%
  \textit{Long-term systemic malaise may be mistaken for the
  hyper-arousal symptoms of a generalized anxiety state.}

  \par\smallskip
  \testlabel
}
&
\cellcolor{unusedbg}
\vcell{%
  \mbox{}

  \par\smallskip
  \textcolor{black!50}{\textbf{UNUSED}}
}
\tabularnewline[-\rowheight]

\cellcolor{conflictbg}
\printcellmiddle
&
\printcelltop
&
\cellcolor{unusedbg}
\printcelltop
\tabularnewline
\bottomrule
\end{tabularx}
\caption{
\textbf{Illustration of the SFT context conditions.}
\textcolor{trainblue}{Training} uses supporting documents from the target or
other domains. \textcolor{testorange}{Testing} uses target-domain documents
that either support the correct answer or conflict with it. This allows us to vary the informativeness of context during training, and evaluate how it may impact models' dependency on context during testing.
}
\label{tab:sft-document-example}
\end{table*}

\autoref{tab:sft-document-example} summarizes these context conditions. During training, supporting documents may come from either the target domain or the paired source domain, whereas evaluation uses only target-domain documents and varies whether they support the correct answer or conflict with it. The conflicting other-domain condition is not used, since the evaluation is intended to isolate the effect of context correctness within the target domain.

Using these generated documents, we construct train-time context-relevance conditions under a fixed budget of eight documents. The variable $k\in\{0,4,8\}$ denotes the number of target-domain
documents, while the remaining $8-k$ documents are drawn from a paired source
domain. We pair Law with Health and Economics with Psychology. Thus, Law uses
Health as the source of irrelevant documents, Health uses Law, Economics uses
Psychology, and Psychology uses Economics. The $k=8$ condition contains only
target-domain context, the $k=4$ condition contains mixed context, and the
$k=0$ condition contains only irrelevant context. We also include a no-context
condition in which all prepended documents are removed. All conditions use the
same question set and answer supervision.

\begin{wraptable}[17]{r}{0.44\textwidth}
\centering
\small
\setlength{\tabcolsep}{3pt}
\renewcommand{\arraystretch}{1.10}

\begin{tabular}{lllccc}
\toprule
Split & Domain & Context & $p_{25}$ & $p_{50}$ & $p_{75}$ \\
\midrule
\multirow{8}{*}{Train}
& Econ. & Supp. & 152.0 & 170.5 & 186.0 \\
& Econ. & Conf. & 148.0 & 163.0 & 176.0 \\
& Health & Supp. & 174.2 & 201.0 & 225.0 \\
& Health & Conf. & 166.0 & 186.0 & 204.0 \\
& Law & Supp. & 173.0 & 198.0 & 219.0 \\
& Law & Conf. & 167.0 & 186.0 & 202.0 \\
& Psych. & Supp. & 157.0 & 176.0 & 192.2 \\
& Psych. & Conf. & 153.0 & 168.0 & 183.0 \\
\midrule
\multirow{8}{*}{Test}
& Econ. & Supp. & 158.0 & 170.5 & 186.2 \\
& Econ. & Conf. & 154.8 & 163.0 & 176.0 \\
& Health & Supp. & 181.0 & 201.0 & 218.0 \\
& Health & Conf. & 173.0 & 190.0 & 208.0 \\
& Law & Supp. & 182.0 & 201.0 & 216.5 \\
& Law & Conf. & 175.0 & 190.0 & 203.0 \\
& Psych. & Supp. & 161.5 & 179.0 & 193.5 \\
& Psych. & Conf. & 155.0 & 168.0 & 186.0 \\
\bottomrule
\end{tabular}
\caption{\textbf{SFT contexts have similar total token lengths.}
Total tokens in the prepended context per question, measured with the
\texttt{Llama-2-7b} tokenizer.
}
\label{tab:sft_context_token_lengths}
\end{wraptable}
As shown in \autoref{tab:sft_context_token_lengths}, supporting and conflicting contexts occupy the same few-hundred-token regime across domains and splits. On the held-out test split, the interquartile ranges largely overlap, where supporting contexts have medians between 170.5 and 201.0 tokens, while conflicting contexts have medians between 163.0 and 190.0 tokens. Supporting contexts are slightly longer overall, but their interquartile ranges remain comparable to those of conflicting contexts within each domain. Thus, the SFT comparison primarily varies the relevance and composition of a fixed eight-document context budget, rather than varying the context length.

\clearpage

\section{Full Results}
\label{app_full_results}

\subsection{Supervised Fine-Tuning Full Results}
\label{app_sft_full_results}

\autoref{tab:full_sft_res} reports the complete supervised fine-tuning
results underlying \autoref{fig:sft_res}. The table breaks down performance
by domain, model size, train-time context condition, and test-time context
condition. Across domains, the same qualitative pattern appears. Increasing the
amount of target-domain context during fine-tuning improves accuracy when
supporting context is available at test time. However, it generally reduces
robustness when context is absent or when the supplied context conflicts with the
correct answer. These domain-level results show that the context-reliance trend in
\autoref{fig:sft_res} is not driven by a single domain or model size.
Instead, it appears consistently across Health, Economics, Law, and Psychology,
as well as across the Qwen3 model scales.

\begin{table*}[h]
\centering
\small
\setlength{\tabcolsep}{4pt}
\renewcommand{\arraystretch}{1.12}
\sisetup{
  table-number-alignment=center,
  round-mode=places,
  round-precision=1
}

\resizebox{\textwidth}{!}{%
\begin{tabular}{
ll
*{25}{S[table-format=2.1]}
}
\toprule
&
& \multicolumn{5}{c}{\textbf{Qwen3-0.6B}}
& \multicolumn{5}{c}{\textbf{Qwen3-1.7B}}
& \multicolumn{5}{c}{\textbf{Qwen3-4B}}
& \multicolumn{5}{c}{\textbf{Qwen3-8B}}
& \multicolumn{5}{c}{\textbf{Qwen3-14B}} \\
\cmidrule(lr){3-7} \cmidrule(lr){8-12} \cmidrule(lr){13-17} \cmidrule(lr){18-22} \cmidrule(lr){23-27}
\textbf{Domain} & \textbf{Eval $\downarrow$ -- Train $\rightarrow$}
& \multicolumn{1}{c}{No-SFT} & \multicolumn{1}{c}{No Ctx.} & \multicolumn{1}{c}{$k=0$} & \multicolumn{1}{c}{$k=4$} & \multicolumn{1}{c}{$k=8$}
& \multicolumn{1}{c}{No-SFT} & \multicolumn{1}{c}{No Ctx.} & \multicolumn{1}{c}{$k=0$} & \multicolumn{1}{c}{$k=4$} & \multicolumn{1}{c}{$k=8$}
& \multicolumn{1}{c}{No-SFT} & \multicolumn{1}{c}{No Ctx.} & \multicolumn{1}{c}{$k=0$} & \multicolumn{1}{c}{$k=4$} & \multicolumn{1}{c}{$k=8$}
& \multicolumn{1}{c}{No-SFT} & \multicolumn{1}{c}{No Ctx.} & \multicolumn{1}{c}{$k=0$} & \multicolumn{1}{c}{$k=4$} & \multicolumn{1}{c}{$k=8$}
& \multicolumn{1}{c}{No-SFT} & \multicolumn{1}{c}{No Ctx.} & \multicolumn{1}{c}{$k=0$} & \multicolumn{1}{c}{$k=4$} & \multicolumn{1}{c}{$k=8$} \\
\midrule

\multirow{3}{*}{Law}
& No Ctx.          & 15.3 & 25.7 & 17.5 & 14.2 & 13.7 & 21.9 & 27.3 & 27.3 & 25.7 & 23.0 & 25.1 & 41.0 & 38.8 & 31.1 & 29.0 & 31.7 & 36.6 & 37.7 & 33.9 & 36.6 & 38.3 & 48.1 & 46.4 & 43.2 & 38.8 \\
& Supporting Ctx.  & 41.0 & 36.1 & 35.5 & 65.0 & 71.0 & 57.9 & 49.7 & 51.9 & 72.1 & 75.4 & 65.0 & 65.6 & 68.3 & 80.9 & 83.1 & 73.8 & 75.4 & 76.5 & 84.2 & 85.8 & 76.5 & 78.7 & 77.6 & 85.8 & 86.3 \\
& Conflicting Ctx. & 10.9 & 25.1 & 19.7 &  8.2 &  6.0 & 13.1 & 21.3 & 22.4 &  6.6 &  4.9 & 14.2 & 21.9 & 23.0 &  4.4 &  4.9 & 14.2 & 17.5 & 20.2 &  4.4 &  3.8 & 13.1 & 21.3 & 24.0 &  5.5 &  4.4 \\
\midrule

\multirow{3}{*}{Health}
& No Ctx.          & 20.6 & 25.5 & 24.8 & 23.4 & 16.3 & 39.7 & 36.9 & 37.6 & 36.9 & 35.5 & 54.6 & 61.0 & 58.2 & 56.7 & 53.9 & 63.1 & 63.1 & 62.4 & 63.8 & 61.0 & 73.0 & 68.8 & 68.1 & 68.8 & 66.7 \\
& Supporting Ctx.  & 75.9 & 58.9 & 48.2 & 80.9 & 83.7 & 86.5 & 77.3 & 83.0 & 90.1 & 87.2 & 94.3 & 90.8 & 92.2 & 95.7 & 95.7 & 96.5 & 93.6 & 91.5 & 95.0 & 95.7 & 95.0 & 92.2 & 92.2 & 95.7 & 95.7 \\
& Conflicting Ctx. & 10.6 & 14.9 & 23.4 &  5.0 &  2.8 & 18.4 & 29.8 & 24.8 & 22.8 & 14.9 & 31.2 & 36.9 & 34.8 & 14.9 & 16.3 & 36.9 & 41.8 & 45.4 & 18.4 & 18.4 & 36.9 & 46.1 & 43.3 & 27.7 & 24.8 \\
\midrule

\multirow{3}{*}{Economics}
& No Ctx.          & 30.3 & 34.9 & 36.8 & 31.6 & 30.3 & 51.3 & 47.4 & 48.0 & 44.1 & 43.4 & 61.8 & 58.6 & 57.2 & 55.9 & 55.9 & 66.4 & 63.8 & 63.2 & 63.2 & 60.5 & 74.3 & 72.4 & 73.0 & 70.4 & 67.1 \\
& Supporting Ctx.  & 53.9 & 43.4 & 39.5 & 63.8 & 59.2 & 69.7 & 66.4 & 58.6 & 74.3 & 71.7 & 80.9 & 77.6 & 75.7 & 81.6 & 80.9 & 82.9 & 75.0 & 75.0 & 81.6 & 82.2 & 83.6 & 84.2 & 84.9 & 84.9 & 84.9 \\
& Conflicting Ctx. & 19.7 & 32.2 & 27.0 & 17.8 & 14.5 & 36.8 & 36.2 & 38.2 & 26.3 & 24.3 & 51.3 & 50.0 & 48.0 & 38.8 & 34.2 & 58.6 & 50.7 & 50.0 & 40.1 & 38.8 & 57.2 & 56.6 & 55.9 & 48.7 & 47.4 \\
\midrule

\multirow{3}{*}{Psychology}
& No Ctx.          & 27.3 & 43.2 & 35.3 & 35.3 & 32.4 & 46.0 & 51.8 & 49.6 & 50.4 & 46.0 & 58.3 & 62.6 & 63.3 & 61.9 & 57.6 & 70.5 & 66.9 & 66.9 & 64.7 & 65.5 & 68.3 & 66.9 & 71.2 & 66.9 & 61.2 \\
& Supporting Ctx.  & 71.9 & 60.4 & 67.6 & 82.0 & 81.3 & 81.3 & 79.9 & 80.6 & 88.5 & 84.2 & 88.5 & 91.4 & 88.5 & 92.1 & 91.4 & 89.2 & 89.9 & 93.5 & 94.2 & 95.0 & 86.3 & 91.4 & 90.6 & 92.8 & 92.1 \\
& Conflicting Ctx. & 20.1 & 35.3 & 32.4 & 12.9 & 12.2 & 31.7 & 41.7 & 43.2 & 34.5 & 26.6 & 41.7 & 51.1 & 50.4 & 30.9 & 33.8 & 39.6 & 48.9 & 50.4 & 31.7 & 33.1 & 40.3 & 49.6 & 48.9 & 38.8 & 37.4 \\
\midrule

\multirow{3}{*}{Average}
& No Ctx.          & 23.4 & 32.3 & 28.6 & 26.1 & 23.2 & 39.7 & 40.8 & 40.6 & 39.3 & 37.0 & 50.0 & 55.8 & 54.4 & 51.4 & 49.1 & 57.9 & 57.6 & 57.6 & 56.4 & 55.9 & 63.5 & 64.1 & 64.7 & 62.3 & 58.4 \\
& Supporting Ctx.  & 60.7 & 49.7 & 47.7 & 72.9 & 73.8 & 73.9 & 68.3 & 68.5 & 81.2 & 79.6 & 82.2 & 81.3 & 81.2 & 87.6 & 87.8 & 85.6 & 83.5 & 84.1 & 88.7 & 89.7 & 85.4 & 86.6 & 86.3 & 89.8 & 89.8 \\
& Conflicting Ctx. & 15.3 & 26.9 & 25.6 & 11.0 &  8.9 & 25.0 & 32.2 & 32.2 & 22.6 & 17.7 & 34.6 & 40.0 & 39.0 & 22.2 & 22.3 & 37.3 & 39.7 & 41.5 & 23.6 & 23.5 & 36.9 & 43.4 & 43.0 & 30.2 & 28.5 \\
\bottomrule
\end{tabular}%
}
\caption{Performance across domains, context conditions, and retrieval settings for Qwen3 models. Average denotes the mean over Law, Health, Economics, and Psychology.}
\label{tab:full_sft_res}
\end{table*}

\clearpage

\subsection{Solution Complexity Full Results}
\label{app_comp_full_results}

\begin{figure}[h]
    \centering
    \includegraphics[width=\linewidth]{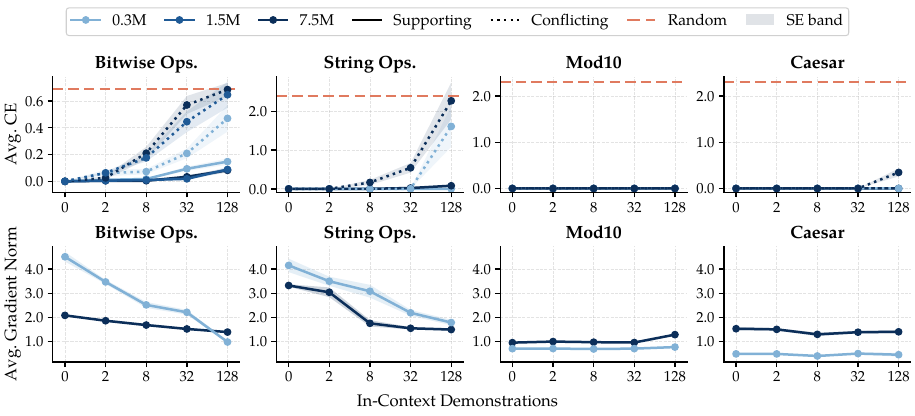}
    \caption{\textbf{Lower complexity solutions and context addiction co-occur across model sizes.} Across synthetic tasks and model scales, tasks with larger supporting--conflicting gaps also show lower average training gradient norms as the number of in-context demonstrations increases. This pattern is strongest for bitwise and string operations, while mod10 arithmetic and Caesar cipher remain comparatively stable. Context addiction becomes stronger as model size grows.}
    \label{fig_complexity_full_res}
\end{figure}

\autoref{fig_complexity_full_res} provides the full model-scale breakdown for the
complexity analysis (\S\ref{sec_solution_complexity}). The same qualitative pattern holds across
model sizes. Tasks that exhibit stronger context addiction also show decreasing
average training gradient norms as train-time demonstrations increase. For
bitwise and string operations, longer contexts produce lower gradient norm
solutions and larger supporting--conflicting gaps. In contrast, mod10 arithmetic and Caesar
cipher show comparatively stable gaps and no consistent decrease in the
gradient norm proxy.

The effect also strengthens with model size. Larger models show larger
growth in the supporting--conflicting gap under longer train-time context. This
suggests that increased capacity does not prevent context addiction in these
settings. Instead, when demonstrations provide a lower complexity route to
reducing loss, larger models appear even more able to exploit that contextual
solution. These results support our conclusion that lower complexity
solutions and context addiction emerge together, and show that this relationship
is consistent across model scales.

\subsection{Token-Level Attention Allocation Full Results}
\label{sec_attn_all}

This section extends the analysis of \S\ref{sec_attn} (\autoref{fig:attn_mass_qwen1}) to the other Qwen3 sizes. As before, we compare models fine-tuned without context against models fine-tuned with $k=8$ task-relevant documents, both evaluated under supporting context. Every model size shows the same pattern: task-relevant fine-tuning increases the attention mass on context tokens, with the largest differences in middle layers (Figures~\ref{fig:attn_mass_qwen3_0_6}--\ref{fig:attn_mass_qwen3_14}).

\begin{figure}[h]
    \centering
    \includegraphics[width=\linewidth]{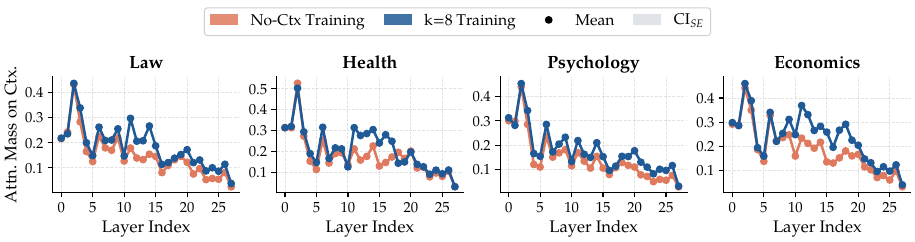}
    \caption{\textbf{Attention allocation for Qwen3-0.6B.} Fine-tuning with task-relevant context increases attention mass on context tokens, especially in middle layers.}
    \label{fig:attn_mass_qwen3_0_6}
\end{figure}

\begin{figure}[h]
    \centering
    \includegraphics[width=\linewidth]{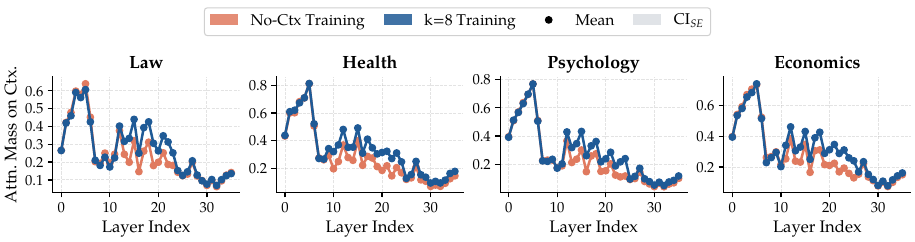}
    \caption{\textbf{Attention allocation for Qwen3-8B.} The context-trained model assigns more attention to context tokens than the no-context model under supporting-context evaluation.}
    \label{fig:attn_mass_qwen3_8}
\end{figure}

\begin{figure}[h]
    \centering
    \includegraphics[width=\linewidth]{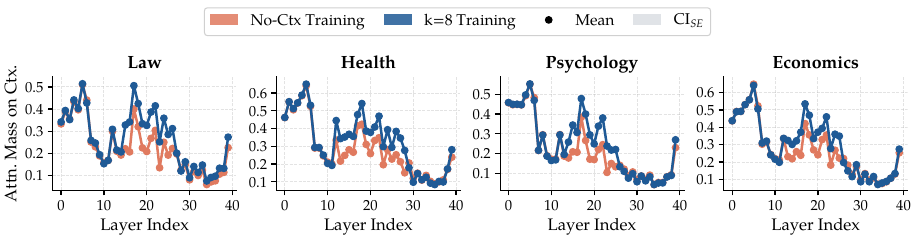}
    \caption{\textbf{Attention allocation for Qwen3-14B.} Task-relevant fine-tuning shifts inference-time attention toward supplied context across tasks, with strongest effects in middle layers.}
    \label{fig:attn_mass_qwen3_14}
\end{figure}
\clearpage
\subsection{Supervised Fine-Tuning Results in the Training Format}
\label{app_format_matched_results}

\autoref{tab:format_matched} reports the domain-averaged results of the training-format evaluation described in App.~\ref{app:format-matched}, alongside the corresponding no-context results from \autoref{tab:full_sft_res}.

\begin{table}[h]
\centering
\small
\setlength{\tabcolsep}{4pt}

\begin{tabular}{lrrrrrrr}
\toprule
& \multicolumn{3}{c}{\textbf{No Ctx.}} & \multicolumn{3}{c}{\textbf{$k=0$ Format}} & \\
\cmidrule(lr){2-4}\cmidrule(lr){5-7}
\textbf{Model} & $k=0$ & $k=4$ & $k=8$ & $k=0$ & $k=4$ & $k=8$ & $\Delta$ \\
\midrule
0.6B & 28.6 & 26.1 & 23.2 & 29.1 & 26.3 & 23.1 & $-6.0$ \\
1.7B & 40.6 & 39.3 & 37.0 & 42.3 & 39.7 & 37.2 & $-5.1$ \\
4B   & 54.4 & 51.4 & 49.1 & 54.9 & 51.1 & 48.7 & $-6.2$ \\
8B   & 57.6 & 56.4 & 55.9 & 59.1 & 56.6 & 55.4 & $-3.7$ \\
14B  & 64.7 & 62.3 & 58.4 & 65.7 & 62.1 & 58.2 & $-7.5$ \\
\bottomrule
\end{tabular}
\caption{\textbf{Matching the $k=0$ training format does not recover accuracy.} Domain-averaged accuracy of Qwen3 models fine-tuned with $k\in\{0,4,8\}$ target-domain documents, evaluated without context (from \autoref{tab:full_sft_res}) and with eight irrelevant documents from the paired domain (App.~\ref{app:format-matched}). $\Delta$ is the change from $k=0$ to $k=8$ in the $k=0$ format.}
\label{tab:format_matched}
\end{table}

\clearpage

\section{Proof of Parametric Information Monotonicity}
\label{app_theory_details}

Let \(\ell\) denote the pointwise task loss and
\(\Risk_k(\Pi,q_k)=\mathbb E[\ell(Y,\hat Y)]\) denote the corresponding
population risk. The risk threshold \(\rho\) denotes an upper bound on this
population risk. We reuse the notation from \S\ref{sec_theory}, where
\(X^{(k)}\) denotes the input available at context size \(k\),
\(T_{k,m}\) is the measurable projection from a longer input \(X^{(m)}\) to its
aligned \(k\)-token subwindow, \(\Pi(\cdot\mid\tau)\) is the
task-conditional weight channel, and \(q_k\) is the predictor at context size \(k\), inducing predictions \(\hat Y\sim q_k(\cdot\mid X^{(k)},W)\).

\begin{proof}
Fix \(k<m\), and take any feasible pair \((\Pi,q_k)\) for
\(\mathcal I_k(\rho)\), so that
\[
\Risk_k(\Pi,q_k)\le \rho.
\]
We construct a feasible pair at context size \(m\) by using the same weight channel
and defining the longer-context predictor by first projecting the input to the
corresponding \(k\)-token subwindow:
\[
q_m(\hat y\mid X^{(m)},W)
=
q_k(\hat y\mid T_{k,m}(X^{(m)}),W).
\]
By the nested-input assumption, \(X^{(k)}=T_{k,m}(X^{(m)})\) almost surely.
Therefore, under the constructed pair \((\Pi,q_m)\), the induced joint law of
\((X^{(k)},Y,W,\hat Y)\) is the same as under \((\Pi,q_k)\). Hence the
population risk is preserved:
\[
\Risk_m(\Pi,q_m)=\Risk_k(\Pi,q_k)\le \rho.
\]
Moreover, because the weight channel is unchanged, the task information stored
in the weights is also unchanged. Thus every feasible pair at context size \(k\) induces a feasible pair at context
size \(m\) with the same parametric information. Taking the infimum over all
feasible \((\Pi,q_k)\) gives
\[
\mathcal I_m(\rho)\le \mathcal I_k(\rho).
\]
\vspace{-1em}
\end{proof}

\section{Additional Related Work}
\label{app_add_rel_work}
\paragraph{Two modes of learning with context.}
A growing body of work studies how language models balance parametric and
contextual strategies for solving tasks. This distinction is often framed as
\emph{task retrieval}, in which demonstrations activate task structure already
encoded in the model's weights~\citep{pan2023what,lin2024dual}, versus
\emph{task learning}, in which the model infers a new input-output rule from
examples in context~\citep{pan2023what,ciphers}. These modes can coexist and compete during pretraining~\citep{wang2023investigating}, and their relative use
depends on the training process, task distribution, and informativeness of
the context~\citep{anand2025dual,chan2022transformers,raventos2023pretraining}. 
Empirical evidence further supports this view. Corrupted-label demonstrations
can still improve performance, suggesting that demonstrations may identify a
latent task rather than fully specify a new mapping~\citep{min2022rethinking}.
Other probes measure when models retrieve internal knowledge versus learn from
demonstrations in regression tasks~\citep{nafar2025learning},
substitution-cipher tasks~\citep{ciphers}, and many-shot prompting
settings~\citep{bertsch2025context}.
These studies establish that language models can interpolate between parametric and contextual solutions. However, they primarily vary inference-time prompts, demonstration distributions, task families, or model scale. The role of the \emph{training context window} in governing this bi-modal tradeoff remains largely unstudied, motivating our focus in this work.

\paragraph{Language modeling as compression.}
The idea that prediction, compression, and intelligence are linked predates
modern LMs. Classical source coding connects probabilistic prediction to
lossless coding~\citep{shannon1948mathematics}, while the
MDL frames learning as compression without overfitting~\citep{rissanen1978modeling,barron1998minimum}. This view also motivates compression-based evaluations of AI models, from text-compression tests as intelligence benchmarks~\citep{mahoney1999text} to the Hutter Prize's Wikipedia-compression objective~\citep{hutter2006prize}. In machine
learning, description length and prequential coding have been used to measure
what neural models learn~\citep{hinton1993keeping,blier2018description} and how
efficiently representations support downstream labels~\citep{
voita2020information,bornschein2022sequential}.
Recent work revisits these ideas for modern language models. 
\citet{deletang2024language} demonstrate that autoregressive LMs can be paired
with entropy coding to form strong general-purpose compressors, and use this lens
to study scaling, tokenization, and in-context learning. Compression also serves as an empirical diagnostic, correlating with downstream performance~\citep{huang2024compression} and guiding data selection~\citep{yin2024entropy}. Relatedly, \citet{elmoznino2024incontext}
interprets in-context learning through prequential
coding, showing that next token prediction can favor simple predictors inferred
from examples in the prompt. Practical neural compressors similarly build on LM predictions, including LLMZip~\citep{
valmeekam2023llmzip}, FineZip~\citep{mittu2024finezip}, and low-complexity
learned compressors such as L3TC~\citep{zhang2025l3tc}. These works establish
compression as an evaluation criterion, a data-selection signal, or a mechanism
for building compressors. Our use of compression is different, since we ask how the
availability of context during training changes the \emph{location} of compressed
task information. Rather than measuring whether a model compresses well overall,
we study whether longer train-time contexts shift predictive structure away from
parametric storage and toward context-conditioned computation.

\end{document}